\documentclass[letterpaper]{article} % DO NOT CHANGE THIS
\usepackage{aaai2026}  % DO NOT CHANGE THIS
\usepackage{times}  % DO NOT CHANGE THIS
\usepackage{helvet}  % DO NOT CHANGE THIS
\usepackage{courier}  % DO NOT CHANGE THIS
\usepackage[hyphens]{url}  % DO NOT CHANGE THIS
\usepackage{graphicx} % DO NOT CHANGE THIS
\usepackage{natbib}  % DO NOT CHANGE THIS AND DO NOT ADD ANY OPTIONS TO IT
\usepackage{caption} % DO NOT CHANGE THIS AND DO NOT ADD ANY OPTIONS TO IT
\usepackage{algorithm}
\usepackage{algorithmic}

\usepackage{subcaption}
\usepackage{comment}
\usepackage{multirow}
\usepackage{colortbl}  %彩色表格需要加载的宏包
\usepackage{diagbox}
\usepackage{nicefrac}
\usepackage{amsmath}
\newtheorem{theorem}{Theorem}
\newtheorem{proof}{Proof}

\usepackage{amssymb}
\usepackage{booktabs}

\usepackage{newfloat}
\usepackage{listings}
\DeclareCaptionStyle{ruled}{labelfont=normalfont,labelsep=colon,strut=off} % DO NOT CHANGE THIS
\floatstyle{ruled}
\newfloat{listing}{tb}{lst}{}
\floatname{listing}{Listing}
\title{Exploring Transferability of Self-Supervised Learning by\\ Task Conflict Calibration}
\author{
    Huijie Guo\textsuperscript{\rm 1}\equalcontrib,
    Jingyao Wang\textsuperscript{\rm 1,2}\equalcontrib,
    Peizheng Guo\textsuperscript{\rm 1,2},
    Xingchen Shen\textsuperscript{\rm 1}, 
    Changwen Zheng\textsuperscript{\rm 1,2}\thanks{Corresponding Author.}, 
    Wenwen Qiang\textsuperscript{\rm 1,2,}\thanks{CO-Corresponding Author.}
}
\affiliations{
    \textsuperscript{\rm 1}Institute of Software Chinese Academy of Sciences,\\
    \textsuperscript{\rm 2}University of Chinese Academy of Sciences,\\
    \{guohuijie,wangjingyao2023,guopeizheng2025,xingchen,changwen,qiangwenwen\}@iscas.ac.cn
}

\usepackage{bibentry}
\begin{document}

\maketitle

\begin{abstract}
    In this paper, we explore the transferability of SSL by addressing two central questions: (i) what is the representation transferability of SSL, and (ii) how can we effectively model this transferability? Transferability is defined as the ability of a representation learned from one task to support the objective of another. 
    Inspired by the meta-learning paradigm, we construct multiple SSL tasks within each training batch to support explicitly modeling transferability. Based on empirical evidence and causal analysis, we find that although introducing task-level information improves transferability, it is still hindered by task conflict. To address this issue, we propose a Task Conflict Calibration (TC$^2$) method to alleviate the impact of task conflict. Specifically, it first splits batches to create multiple SSL tasks, infusing task-level information. Next, it uses a factor extraction network to produce causal generative factors for all tasks and a weight extraction network to assign dedicated weights to each sample, employing data reconstruction, orthogonality, and sparsity to ensure effectiveness. Finally, TC$^2$ calibrates sample representations during SSL training and integrates into the pipeline via a two-stage bi-level optimization framework to boost the transferability of learned representations. Experimental results on multiple downstream tasks demonstrate that our method consistently improves the transferability of SSL models.
\end{abstract}

\begin{links}
    \link{Code}{https://github.com/PaulGHJ/TC2}
\end{links}

\section{Introduction}

Traditional unsupervised learning lacks ground-truth labels and often performs worse than supervised learning. However, supervised methods heavily rely on costly manual annotations and still face generalization issues~\cite{bousquet2002stability,mignacco2020role}. As a compelling alternative, Self-Supervised Learning (SSL) has shown strong performance in image classification and object recognition~\cite{schiappa2023self,drmac,liu2022graph}, sometimes even surpassing supervised methods.

Existing self-supervised learning (SSL) methods primarily focus on designing reliable supervisory signals or introducing inductive biases to guide the model in learning effective feature representations. For example, discriminative SSL (D-SSL) methods construct positive pairs by generating different augmented views of the same sample and treat other samples as negatives~\cite{chen2020simple,he2020momentum,caron2020unsupervised}. Generative SSL (G-SSL) methods mask part of the input and train the model to reconstruct the masked regions, using the reconstruction error as the supervisory signal~\citep{he2022masked,baobeit}. However, these methods often overlook a critical question: how do such learning paradigms ensure the transferability of learned features? This question is not easy to answer and helps explain why these methods often perform poorly on out-of-distribution (OOD) data or transfer tasks. Additionally, this limits the theoretical insights and methodological advances needed to improve their transfer generalization capabilities.

In this paper, we focus on two key questions: (1) What is representation transferability? (2) How can it be effectively modeled?
For the first question, we take an intuitive perspective: if a representation has good transferability, it should perform well across various downstream tasks. Based on this intuition, we define task-level transferability as the ability of a representation learned from one task to support the objective of another~\citep{yosinski2014transferable}. To address the second question, we consider meta-learning as an established approach to modeling transferability. Prior works~\cite{ni2021close} suggest that SSL can be regarded as a special case of meta-learning. The main difference lies in task organization within each mini-batch and learning mechanism.
Meta-learning typically involves multiple tasks, whereas SSL is generally considered to involve only a single task. Moreover, meta-learning is optimized in a bi-level manner. It achieves strong transferability through multi-task bi-level training, which implicitly models a distribution over tasks. This allows meta-learning to generalize well to unseen tasks. Motivated by this, we propose enhancing the transferability of SSL by introducing a meta-learning mechanism. Specifically, we construct multiple tasks within each mini-batch during training and train SSL models in a bi-level manner.

\begin{figure*}[t]
    \centering
    \begin{subfigure}[b]{0.265\textwidth}
         \includegraphics[width=\textwidth]{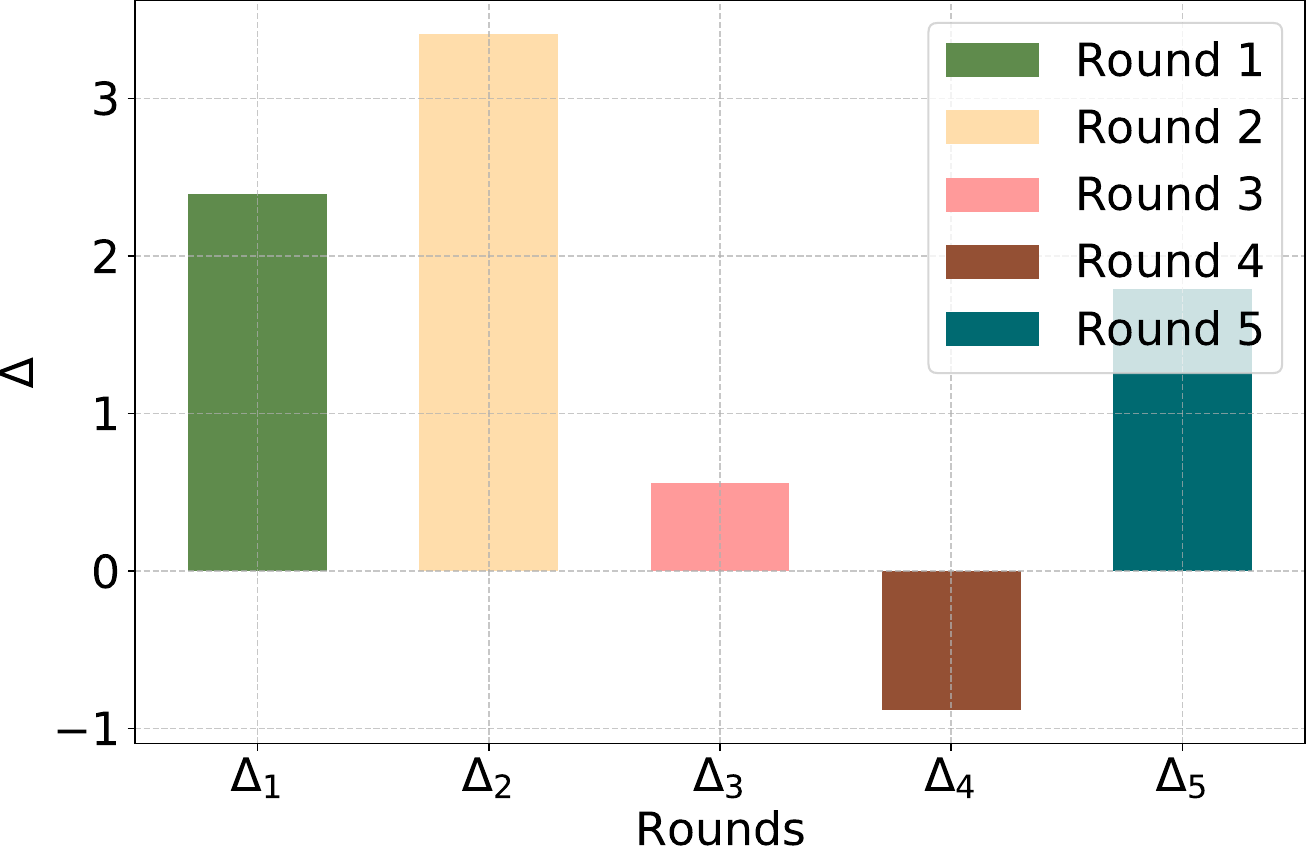}
         \caption{}
         \label{fig:motivation_comparison}
    \end{subfigure}
    \hfill
    \begin{subfigure}[b]{0.30\textwidth}
         \includegraphics[width=\textwidth]{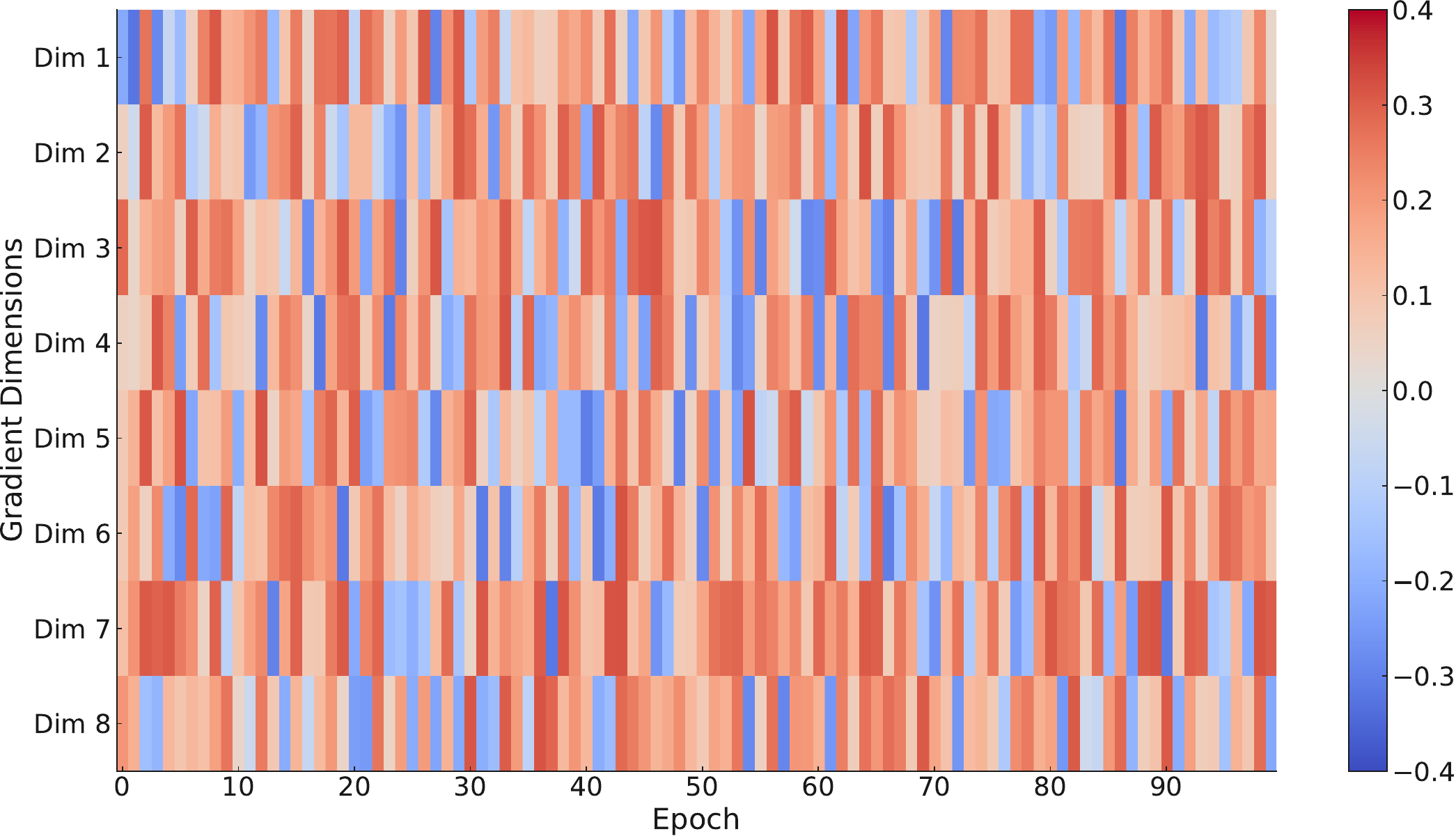}
         \caption{}
         \label{fig:motivation_subfig1}
    \end{subfigure}
    \hfill
    \begin{subfigure}[b]{0.30\textwidth}
         \includegraphics[width=\textwidth]{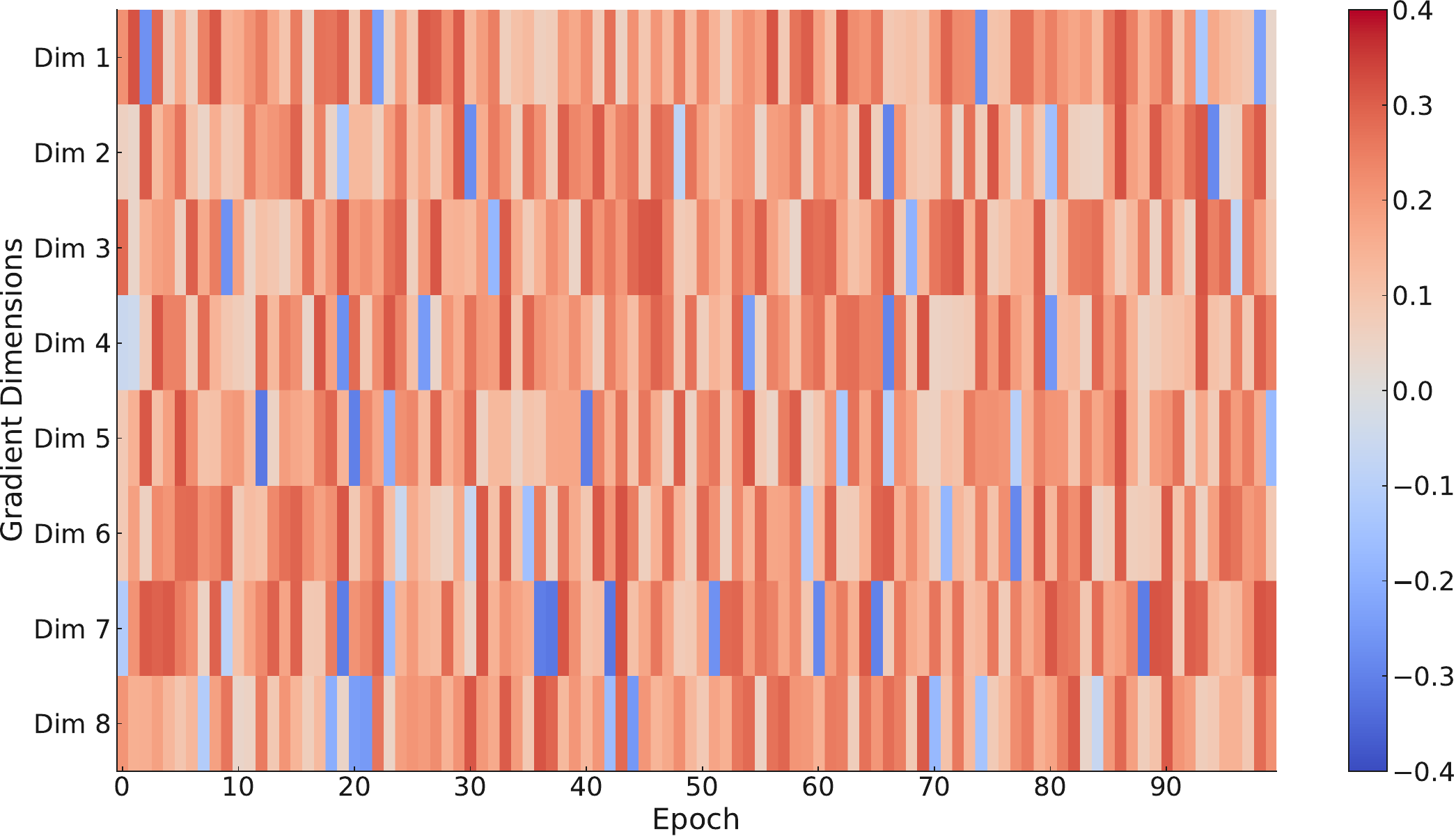}
         \caption{}
         \label{fig:motivation_subfig2}
    \end{subfigure}
    \caption{Effect of task-level information on CIFAR-100. (a) shows the effect of SimCLR trained on ImageNet and tested on CIFAR-100 before and after introducing task-level information in five runs. (b) and (c) show the training process corresponding to the worst and best rounds in five evaluations, visualizing the gradient similarity between tasks.}
    \label{fig:motivation_dimension}
\end{figure*}

We further explore the effectiveness of the above meta-learning-like SSL method on transferability. In the experiment, we uniformly split each training batch into two subsets, which serve as two independent SSL tasks. Following the standard SSL training framework, each subset generates two augmented views via random data augmentation for training. We use SimCLR as the baseline, training on ImageNet-100 and testing on other datasets. First, we evaluate the Top-1 accuracy without task construction. Next, we assess performance with the multi-task learning mechanism through five independent training rounds and calculate the improvement for each round. From Figure~\ref{fig:motivation_comparison}, the results demonstrate that incorporating task-level information significantly improves model performance, but the improvement is not consistent, with a nearly 3\% performance gap between the best and worst outcomes, and some cases even showing negative gains. Further analysis reveals that gradient conflict between tasks during training significantly affects transferability, as shown in Figure \ref{fig:motivation_subfig1} and~\ref{fig:motivation_subfig2}. Therefore, when introducing task-level information enhances transferability, inter-task gradient conflict remains an important issue.

To better understand the underlying causes of task conflict and its impact on transferability, we analyze it through both data generation and causal inference. Specifically, we construct a Structural Causal Model (SCM) to illustrate these relationships (as shown in Figure~\ref{fig:scm_1}), where $Y_i$, $Y_j$ are task labels, $F_{i,j}^s$ is the shared generative factors, and $F_i^u$, $F_j^u$ are task-specific factors. In meta-learning-like SSL, the model estimates generative factors from multiple tasks and predicts labels accordingly. Due to the shared training across tasks, the extracted factors inevitably mix semantics from all tasks(Figure~\ref{fig:scm_2}). As shown in Theorem~\ref{thm1}, if task-specific factors are not correctly identified, these factors from other tasks act as confounders, hindering target task learning—this phenomenon is termed task conflict. It degrades per-task learning and limits the effectiveness of modeling task distributions, explaining the inconsistent transfer performance observed in meta-learning-like SSL. Moreover, these generative factors can be seen as semantic vectors composing the feature space. Since different tasks may involve distinct semantics, task conflict implies that target features may be contaminated by irrelevant semantics. Thus, enhancing SSL transferability requires calibrating each sample’s feature representation to retain only task-relevant semantics within the meta-learning-like framework.

Based on the above analysis, we propose a new method, called Task Conflict Calibration (TC$^2$), to alleviate task conflict when explicitly modeling the transferability for SSL methods. 
First, TC$^2$ constructs multiple SSL tasks through batch splitting, which introduces task-level information into the model. 
Second, TC$^2$ learns a factor extraction network $f_v$ to generate causal generative factors for all tasks and a weight extraction network $f_w$ to generate a weight of a single sample of each task. We leverage techniques such as data reconstruction, orthogonality, and sparsity to ensure that the learned \( f_v \) can effectively generate causal generative factors suitable for multiple tasks, while the learned function \( f_w \) produces a weight matrix of semantic vectors that are specifically relevant to the target task. Third, based on \( f_v \) and \( f_w \), we formulate the final version of TC\(^2\), which serves to calibrate sample representations during the SSL training process. Finally, we embed TC\(^2\) into the SSL training pipeline through a two-stage bi-level optimization framework to enhance the transferability of learned representations. Extensive transfer learning and ablation experiments demonstrate the effectiveness of TC\(^2\).
Our contributions include:
\begin{itemize}
\item  We find that simply formulating the SSL training process as a meta-learning paradigm does not guarantee improved feature transferability. Through causal analysis, we demonstrate this limitation stems from task conflict.
\item  We propose a method to alleviate task conflict and improve the transferability of SSL models, leveraging a two-stage bi-level optimization mechanism.
\item Experimental results on multiple downstream tasks demonstrate the effectiveness of our method while preserving the model’s inherent representation capabilities.
\end{itemize}

\section{Related Work}
\label{r_work}

\textbf{Self-Supervised Learning} (SSL) has demonstrated remarkable performance in multiple fields, even rivaling supervised learning. Early SSL methods focused on designing pretexts~\cite{albelwi2022survey,doersch2015unsupervised} to provide supervisory information for the unlabeled data. With the rise of contrastive learning, SSL gradually turns to learning more generalizable representations based on invariance to data augmentation. Typical contrastive learning constructs positive sample pairs through data augmentation so that positive pairs are close and negative pairs are separated in the embedding space. SimCLR~\cite{chen2020simple} introduces a simple yet effective contrastive learning framework, where the model is trained to maximize the consistency between positive pairs and minimize the consistency between negative pairs through the InfoNCE loss~\cite{gutmann2010noise}. To reduce the computational cost caused by batch size, MoCo~\cite{he2020momentum} uses a memory bank to increase the number of negative samples and a momentum encoder to learn representations. SwAV~\cite{caron2020unsupervised} enforces consistency between cluster assignments from different augmented views of the same image, rather than computing comparisons of sample pairs. BYOL~\cite{grill2020bootstrap} and SimSiam~\cite{chen2021exploring} only constrain the consistency between positive samples to avoid trivial solutions through the online and target networks. Additionally, some studies aim to reduce redundancy among features\cite{liu2022self,bardes2022vicreg}. Barlow Twins~\cite{zbontar2021barlow} endeavors to make the normalized cross-correlation matrix of the augmented embeddings close to the identity matrix. 

\textbf{Transferability} refers to the ability of a model trained on one task to be transferred to another~\cite{pan2009survey,t2mac,PMAE}. 
The interplay between out-of-distribution and transfer performance of convolutional neural networks is investigated by \cite{djolonga2021robustness}.
According to \cite{ying2018transfer}, achieving optimal transferability often requires exhaustive search or the use of diverse task distributions.
Multi-task  learning~\cite{standley2020tasks} makes the model more adaptable by sharing knowledge among multiple related tasks, thereby improving its transferability to new ones. 
The connection between the structure of vision tasks and transferability is highlighted by \citep{zamir2018taskonomy}, which explores which tasks can be effectively transferred to any target task. 
To enable fast adaptation to new tasks with limited labels, meta-learning leverages many small-sample tasks to train a base learner~\cite{hospedales2021meta,immac}.
The connection between SSL and meta-learning to improve the transferability of self-supervised models is explored by some work\cite{ni2021close}.

Existing SSL methods construct models based on samples, which can be treated as an in-distribution problem for a specific task, without explicitly modeling transferability from the task level. 
In this paper, we explicitly model transferability for SSL by constructing multiple tasks and improving the transferability by resolving task conflict issues.

\section{Problem Analysis and Motivation}
\label{Motiavtion}

This section begins by revisiting the self-supervised learning paradigm, highlighting the implicit task structure underlying it. We then explore the inherent connection between self-supervised learning and meta-learning through the lens of task modeling. Building on this, we empirically investigate how introducing task-level information can lead to task conflicts, which hinder the transferability of learned representations. Finally, we analyze the root causes of these conflicts from a causal perspective.

\subsection{Rethinking SSL from a Task Perspective}

In the training phase, data is organized into mini-batches $\mathcal{X}_{tr} = \{x_i\}_{i=1}^N$. In D-SSL methods such as SimCLR~\citep{chen2020simple} and Barlow Twins~\citep{zbontar2021barlow}, each sample is stochastically augmented into two views $(x_i^1, x_i^2)$. In G-SSL methods like MAE~\citep{he2022masked}, $x_i$ is divided into patches, masked, and reassembled into $x_i^1$, while the original sample serves as $x_i^2$. Thus, the augmented training set becomes $\mathcal{X}_{tr}^{aug} = \{(x_i^1, x_i^2)\}_{i=1}^N$, and SSL aims to learn a feature extractor $f$ from these pairs.

D-SSL objectives typically consist of alignment (maximizing similarity within pairs) and regularization (e.g., enforcing uniformity~\citep{wang2020understanding}). G-SSL achieves alignment by reconstructing $x_i^2$ from $x_i^1$ using an encoder-decoder structure. In both cases, one sample in the pair serves as an \emph{anchor}, guiding the other to align with it. For D-SSL, this is explicit; for G-SSL, $x_i^2$ acts as the anchor for reconstructing $x_i^1$. This process can be interpreted as assigning one sample as a learning target and adjusting the other to match it in feature space. As a result, paired samples become tightly clustered, with the anchor acting like a cluster center. Consequently, $\mathcal{X}_{tr}^{aug}$ can be viewed as a mini-batch classification task with $N$ implicit classes, where alignment functions similarly to class-specific feature learning.

\subsection{Comparsion of SSL and Meta-Learning}
\label{sec:meta-learning}
During the meta-training phase of meta-learning, the model is trained on a collection of tasks sampled from a task distribution. Each training mini-batch consists of multiple tasks $\{ \tau_1, \tau_2, \cdots, \tau_B \}$, where each task $\tau_i$ contains a support set $S_i$ and a query set $Q_i$. The support set is used to adapt the model to the specific task, while the query set is used to evaluate and update the meta-learner’s parameters based on the adaptation performance. The learning mechanism of meta-learning typically follows a bi-level optimization paradigm~\cite{gordon2018meta,jiang2022role}: the inner loop adapts task-specific parameters using the support set, and the outer loop updates the shared meta-parameters using the query set loss across all tasks in the mini-batch. This enables the model to acquire generalizable knowledge that can be quickly adapted to unseen tasks during meta-testing.

Comparing SSL and meta-learning, we can obtain SSL and meta-learning share certain similarities in task construction but differ significantly in mini-batch construction and learning mechanism \cite{ni2021close}. Based on the above analysis, we can obtain the collection of all instance pairs across the entire training set can be viewed as forming a multi-class classification task. Similar to traditional machine learning, the training process of SSL can be viewed as a mini-batch-based learning procedure. The basic unit of its training set is a sample pair, and multiple distinct pairs form a mini-batch. The training process essentially aims to approximate the conditional distribution $p(y|x)$. Notably, the learned distribution $p(y|x_{\text{train}})$ can generalize to $p(y|x_{\text{test}})$ only under the assumption that training and test data are i.i.d.. However, in OOD or transfer tasks, this i.i.d. assumption often fails, which poses a challenge to the performance of SSL in such scenarios. In contrast, meta-learning adopts a different structural unit: each basic unit in the training set corresponds to an entire task, and a mini-batch consists of multiple such tasks. During meta-training, the model learns to approximate the same $p(y|x)$ across all samples within any given task. Since both training and test tasks are assumed to be sampled from the same (possibly unknown) task distribution, the i.i.d. assumption naturally holds at the task level. As a result, the function learned during meta-training can be effectively transferred to test tasks, making meta-learning more robust and reliable when facing OOD problem or transfer learning scenarios.

\subsection{Empirical Evidence}

One promising way to improve the transferability of SSL representations is to incorporate task-level information inspired by meta-learning. Specifically, we uniformly split each training batch $\mathcal{X}_{tr}^{aug}$ into two subsets \(B_1\) and \(B_2\), which serve as two distinct tasks. Each task is further divided into a support set \(B_i^s\) and query set \(B_i^q\), where \(i \in \{1, 2\}\), following the meta-learning framework.
We adopt SimCLR as the baseline, training on ImageNet-100 and testing on multiple datasets. Here we just report results on CIFAR-100(Details are provided in the Appendix). First, we evaluate the Top-1 accuracy without task construction, denoted as \(acc(\mathrm{SimCLR})\). Then, using the task-based multi-task learning mechanism, we conduct five independent runs, with the \(i\)-th result denoted as \(acc(\mathrm{SimCLR+T}, i)\). The corresponding performance gain is \(\Delta_i = acc(\mathrm{SimCLR+T}, i) - acc(\mathrm{SimCLR})\). Results in Figure~\ref{fig:motivation_comparison} show that task-level design improves performance, but with noticeable instability—performance differences can reach nearly 3\%, and negative gains are observed in some runs.

To investigate this instability, we analyze the first 100 training epochs of the best and worst performing runs. We compute the cosine similarity of gradients between tasks at each epoch. A negative value indicates conflicting gradients (angle $>90^\circ$), while a positive value indicates alignment. As shown in Figure~\ref{fig:motivation_subfig1} and Figure~\ref{fig:motivation_subfig2}, we observe that: (i) gradient similarity fluctuates during training, often exhibiting conflict; (ii) lower inter-task gradient conflict correlates with better transfer performance. These results suggest that while task-level information enhances SSL transferability, it introduces inter-task gradient conflict, which can negatively affect consistency and stability in performance.

\subsection{Causal Analysis and Motivation}
\label{task conflict}

\textbf{Causal Analysis.} To investigate the underlying reasons for the above phenomenon, we propose using causal theory for analysis. We construct a Structural Causal Model (SCM) based on the causal generating mechanism~\cite{suter2019robustly,hu2022improving}, as shown in Figure \ref{fig:intro_SCM}(a). Specifically, the SCM comprises two tasks, \(\tau_i\) and \(\tau_j\), where \(Y_i\) and \(Y_j\) denote the label variables for tasks \(\tau_i\) and \(\tau_j\), and \(X_i\) and \(X_j\) represent the generated samples corresponding to these tasks. Meanwhile, \(F^u_i\) and \(F^u_j\) denote the sets of unique causal factors specific to tasks \(\tau_i\) and \(\tau_j\), respectively, and \(F_{i,j}^s\) contains the shared causal generative factors.
Following \cite{sun2024m2i2,deshpande2022deep}, the samples are generated with all the causal generative factors, e.g., for \(\tau_i\), the sample \(X_i\) is generated by both the task-specific factors \(F^u_i\) and the shared factors \(F_{i,j}^s\). Since each causal generative factor in \(F^u_i\), \(F^u_j\), and \(F_{i,j}^s\) corresponds to a specific semantic attribute (e.g., color, shape) in its respective task, they can be viewed as being determined by the label variable \(Y_i\) (or \(Y_j\)) of the tasks.
Thus, we obtain Figure~\ref{fig:scm_1}, and for task \(\tau_i\), we refer to \(F_{i,j}^s\) and \(F^u_i\) as the causal generating factors that are causally related to \(Y_i\), while \(F^u_j\) is considered non-causal factors.

\begin{figure}[t]
    \centering
    \begin{subfigure}[b]{0.36\linewidth}
         \includegraphics[width=\linewidth]{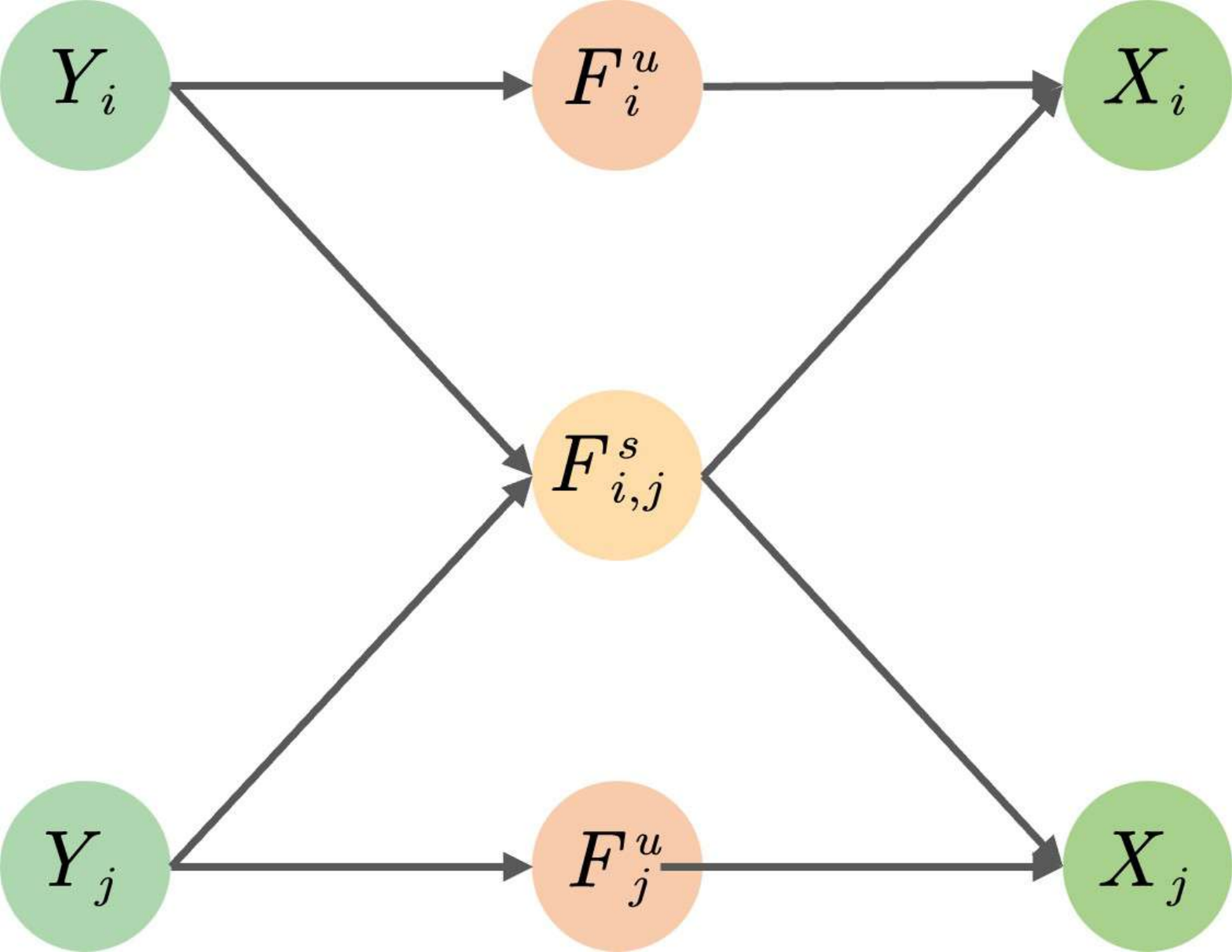}
         \caption{}
         \label{fig:scm_1}
    \end{subfigure}
    \hfill
    \begin{subfigure}[b]{0.36\linewidth}
         \includegraphics[width=\linewidth]{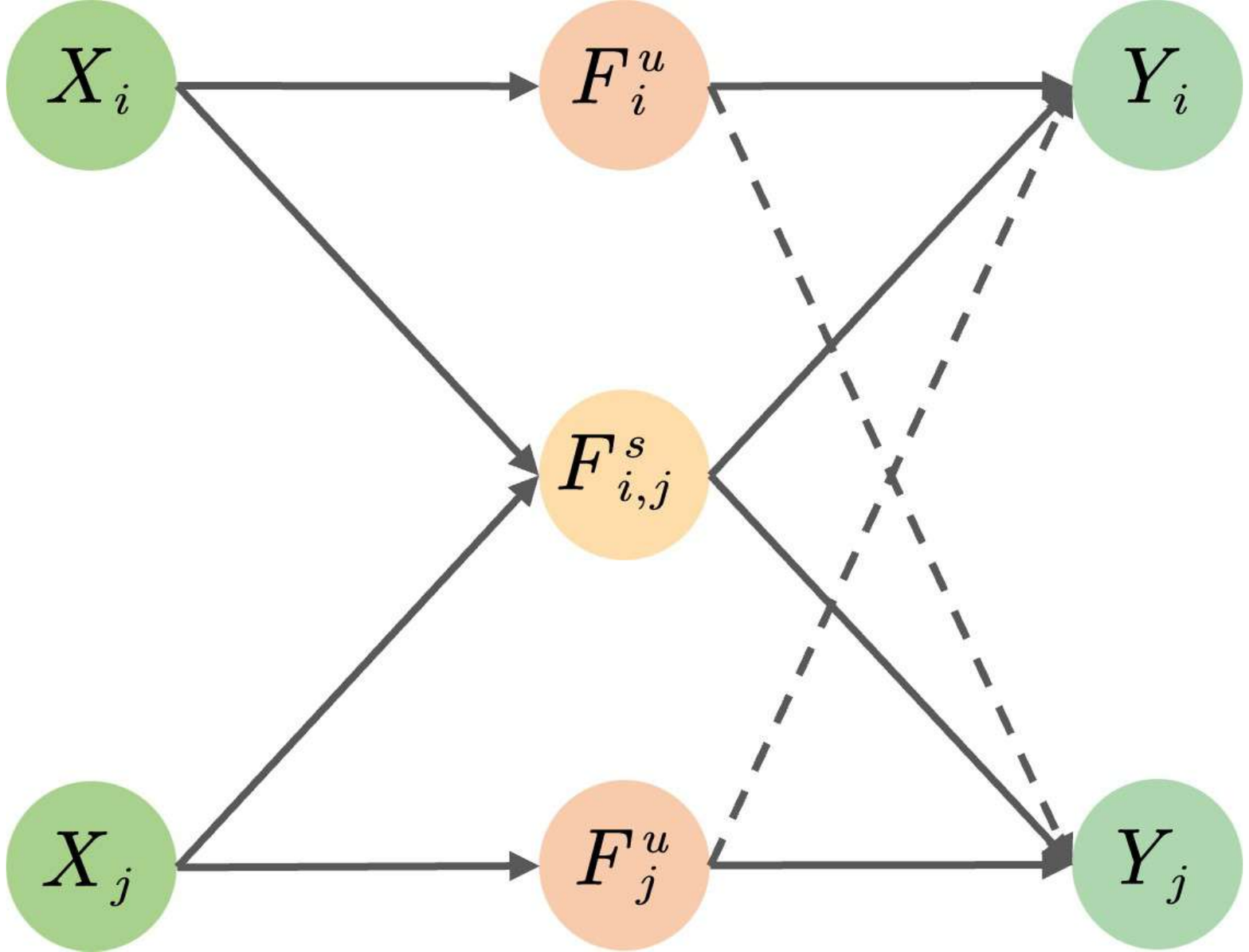}
         \caption{}
         \label{fig:scm_2}
    \end{subfigure}
    \caption{(a) Overview of the SCM based on generation mechanisms; (b) The causal mechanism of task conflict.}
    \label{fig:intro_SCM}
\end{figure}

Based on the above SCM, a good SSL model should utilize only the causal factors to learn the tasks, thereby achieving accurate decisions even in the absence of labels. However, to achieve modeling transferability, a multi-task joint learning mechanism is required, which causes the model to learn all the causal factors across tasks, including \(F^u_i\), \(F^u_j\), and \(F_{ij}^s\). Consequently, for a specific task \(\tau_i\), the model may learn from non-causal factors originating from other tasks, which poses a significant challenge to attaining optimal predictions. To validate this assertion, we consider a scenario involving two tasks in the same SSL training batch. Here, we treat the pseudo-labels generated by data augmentation as \(Y_i\) and \(Y_j\).
Then, we have:
\begin{theorem} \label{thm1}
If the correlation between $Y_i$ and $Y_j$ is not equal to 0.5, then the optimal model for task $i$ has a non-zero weight on $F^u_j$. If the correlation is equal to 0.5 with limited training samples, then the optimal classifier for task $i$ also has non-zero weight on factor $F^u_j$.
\end{theorem}
It shows that the learned SSL model leverages causal factors from other tasks to facilitate the learning of the target task (with proofs in Appendix). For example, in task \(\tau_i\), the model employs causal factors \(F^u_j\) from task \(\tau_j\) to learn \(Y_i\), creating a spurious correlation represented as \(F^u_j \to Y_i\); similarly, \(F^u_i \to Y_j\) exists for task \(\tau_j\). Therefore, the SCM for the learning phase, shown in Figure~\ref{fig:scm_2}, which contains two spurious paths.
These spurious correlations lead to suboptimal learning for each task. For any target task, interference from other tasks during training disrupts learning. A direct indicator of this is when the gradients between tasks have angles greater than \(90^\circ\) (Figure~\ref{fig:motivation_dimension}), meaning that updates from other tasks suppress the gradient updates of the target task—a phenomenon we refer to as ``task conflict''.

\textbf{Motivation.} 
Based on the above causal analysis, enhancing SSL feature transferability requires not only constraining the learning mechanism as in meta-learning but also addressing task conflict during training. Generative factors can be viewed as semantic vectors~\cite{pearl2009causality, wolff2019vector,scholkopf2021toward}, and sample features as projections onto them. Task conflict arises when the multi-task learning process forces the model to learn all task-related semantic vectors. As a result, a target task's feature representation may contain non-zero projections onto irrelevant semantics from other tasks, degrading its effectiveness.
To resolve this, we propose identifying the semantic vectors relevant to the target task and retaining only the corresponding feature projections. This approach effectively eliminates the interference of unrelated semantics, thereby mitigating task conflict and improving feature transferability.

\section{Methodology}
\label{Sec:Method}

\subsection{Task Construction}
\label{Sec:TC}
In this subsection, we describe how to construct SSL tasks, as shown in Figure \ref{fig:pipeline}. We begin by randomly sampling a mini-batch $X_{tr}=\{x_i \}_{i=1}^N$ from the training set. 

Drawing inspiration from meta-learning \cite{vilalta2002perspective, finn2017model}, we divide the mini-batch $X_{tr}$ into $K$ groups, denoted as $B_1,...,B_k,...B_K$, with each group treated as an individual task. This grouping strategy allows the model to simulate diverse tasks within a single batch, promoting the learning of task-agnostic representations. Such representations are crucial for improving the model’s transferability to unseen domains or tasks.
For D-SSL, we randomly apply $m$ data augmentation strategies to each group to generate the augmented dataset $B_k^{aug}$, containing $m$ augmented views for each original sample. Each task consists of a support set and a query set. The support set is constructed by randomly selecting $m'$ augmented views from each sample, while the remaining views are assigned to the query set. The support and query sets are then used to update the task-specific model parameters and guide the learning process for each task.
For G-SSL, each sample is first divided into multiple patches. Then, $m$ different random masks are applied to produce $m$ masked views. Each task consists of a support set and a query set. The support set is constructed by randomly selecting $m'$ masked views for each sample, along with the original (unmasked) sample. The remaining masked views are used to construct the query set. The model is trained to reconstruct the original input from the masked views by minimizing the reconstruction loss. 
By constructing multiple tasks, the SSL model can learn both the sample-level and task-level information, improving its transferability. However, the inevitability of task conflicts may affect the transferability of the model (Figure \ref{fig:motivation_dimension}), and we will address this issue in the next subsection.

\begin{figure}
    \centering
    \includegraphics[width=0.88\linewidth]{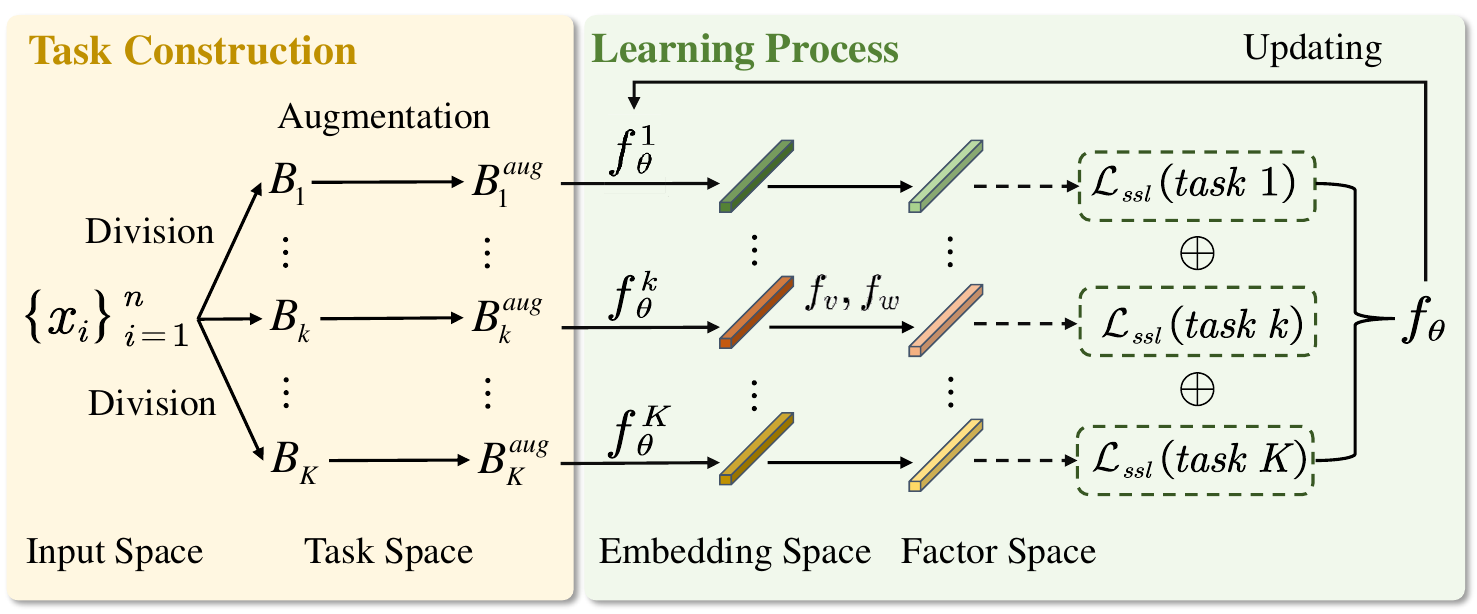}
    \caption{Overview of the proposed framework.}
    \label{fig:pipeline}
\end{figure}

\subsection{Task Conflict Calibration}
\label{Sec:TF}

In this subsection, we propose a method for calibrating  task conflict. Specifically, we address the following three key questions: (i) how to learn the union of causal generative factors across all tasks; (ii) how to identify the task-specific causal factors for individual tasks; and (iii) how to alleviate the conflicts that may arise between tasks. 

For question (i), we first compute the mean vector and covariance matrix of all sample representations in the training batch. The mean vector represents the center of the sample features, while the covariance matrix captures the relationships between the features. We then concatenate these two statistical quantities to form a new matrix $S \in R^{(d+1) \times d}$, for the subsequent factor extraction process. Next, we fed the concatenated matrix into a factor extraction network $f_v$. The output of $f_v$ is constrained to be a factor matrix $V = f_v(S) \in \mathbb{R}^{d_v \times n_f}$, where each column represents the basis of a factor, $d_v$ is the dimension of the generative factor (which is the same as the dimension of the sample representation, denoted as $d$), and $n_f$ is the number of factors. At last, the loss function of learning $f_v$ can be presented as:
\begin{equation} \label{eq:all_factor}
\mathcal{L}_{v}(f_v) =  \left \| V^{\rm T}V - I \right \| _{F}^{2} + \left \| Z_{tr}^{aug} - Z_{tr}^{aug}VV^{\rm T} \right \| _{F}^{2}.
\end{equation}
where $Z_{tr}^{aug}$ represent the augmented dateset on the whole training batch, $I$ is a identity matrix, $\left \| \cdot \right \| _{F}$ is the $F$-paradigm of the matrix. The first term is to constrain the orthogonality of the column vectors of $V$, and the second term is to constrain $V$ to best reconstruct the training dataset.

Since the SSL task in this work is constructed by partitioning the entire training batch, and the full set encompasses all tasks, it is reasonable to base the learning on the entire training batch. This approach allows the model to capture the global characteristics of the data, facilitating generalization and transfer across tasks. 
Specifically, the mean vector and covariance matrix serve as sufficient statistics, summarizing the main features of the data and reflecting its overall distribution. By using these statistics to extract factors, the model learns shared representations that can adapt to different tasks, thereby improving transferability. As different factors represent distinct semantics, enforcing orthogonality among their corresponding vectors serves as a feasible approach to learning $f_v$. Minimizing the reconstruction error ensures that the factor matrix accurately reconstructs the data, allowing the shared factors across tasks to be effective and adaptable, further supporting transfer learning.

For question (ii), given a sample of a task, we input it into a leanrnable weight extraction network $f_w$. We constrain that the output of $f_w$ is a weight vector with $n_f$ dimension. Then, the loss function of learning $f_w$ is presented:
\begin{equation} \label{eq:task_factor}
    \mathcal{L}_{w}(f_w) = \sum_{k}\sum_{z_j \in {\rm task}\, k} \left \| z_j - W_jV^T \right \|_{2}^{2} + \left \| {W_j} \right \|_{1}.
\end{equation}
Here, $V$ is a pre-given matrix, $W_j = f_w(z_j)V^T$ represents the weight vector of the \( j \)-th sample under task \( k \), $z_j$ is the representation of a sample. 
The first term ensures that each sample representation $z_j$ is reconstructed by a weighted combination of a few shared factors in $V$, with $W_j$ serving as its sample-specific weight vector. This encourages the model to capture task-specific structures through selective factor usage.
The second term applies an $l_1$-norm regularization to $W_j$, promoting sparsity by activating only a limited number of factors. This improves interpretability and reduces overfitting.Together, these terms guide $f_w$ to assign sparse, task-aware weights over shared factors, enabling the identification of task-specific causal components.

The essence of task conflict lies in the causal generative factors of the target task being confounded by factors from other tasks. A feasible solution is to constrain the features of samples from the target task to contain only semantics that are relevant to the target task itself. From the perspective of feature representation, this can be interpreted as enforcing that the features of a sample are generated solely by the causal generative factors associated with the target task. Therefore, Bases on the solutions to question (i) and (ii), the solution to question (iii) for any sample $z_j$ in the target task can be directly formulated as: $f_w(z_j)V^T$.
Finally, we get the task conflict calibration loss by combining Eq.\ref{eq:all_factor} and Eq.\ref{eq:task_factor}:
\begin{equation}
\label{eq:ldf}
    \mathcal{L}_{tc^2}(f_v, f_w) =  \mathcal{L}_{v} + \lambda_s \mathcal{L}_{w},
\end{equation}
where $\lambda_s$ is the parameter used for balancing two terms.

\subsection{Learning Process}
\label{APP_F&A}

The training process with TC$^2$ in each batch is performed via two stages. In this first stage, we fix the SSL model and optimize the transfer function $f_v$ and weight function $f_w$ via bi-level optimization. In the second stage, we optimize the SSL model with $f_v$ and $f_w$ being held.

In the first stage, we optimize $f_v$ and $f_w$ to ensure the causality of the learned generative factors. Following~\cite{zhu2021understanding,deng2022strong}, generative factors should remain invariant across different sample distributions within a task. An effective model captures these factors and adapts across tasks. Support and query sets from the same task represent shifted distributions, and the meta-learning objective is to learn invariants effective across both.

Thus, we adopt bi-level optimization to learn $f_v$ and $f_w$, enforcing causal invariance. First, we first update $f_v$ and $f_w$ on the support sets of all tasks by the following formula:
\begin{equation}
\label{eq:wv_inner}
\begin{array}{c}
f_v'\leftarrow f_v-\alpha_1\nabla_{f_v}\bar{\mathcal{L}}^{s}, \quad f_w'\leftarrow f_w-\alpha_2\nabla_{f_w}\bar{\mathcal{L}}^{s} \\
s.t. \bar{\mathcal{L}}^{s} = \frac{1}{K}\sum\limits_{k=1}^{K} \mathcal{L}_{ssl}^{s}(f_w, f_v, B^{aug}_k) + \mathcal{L}_{tc^2}(f_v, f_w)
\end{array}
\end{equation}
Here, $\mathcal{L}_{ssl}^{s}(\cdot)$ is the self-supervised loss on support set $S_k$ of task $k$, and $\alpha_1$, $\alpha_2$ are learning rates. Note that losses are computed in the factor space.

Next, $f_v$ and $f_w$ are updated on query sets:
\begin{equation}
\label{eq:wv_outer}
\begin{array}{c}
f_v \leftarrow f_v' - \alpha_3\nabla_{f_v'}\bar{\mathcal{L}}^{q}, \quad f_w \leftarrow f_w' - \alpha_4\nabla_{f_w'}\bar{\mathcal{L}}^{q} \\
s.t. \bar{\mathcal{L}}^{q} = \frac{1}{K}\sum\limits_{k=1}^{K} \mathcal{L}_{sse}^{q}(f_w', f_v', B^{aug}_k) + \mathcal{L}_{tc^2}(f_v', f_w')
\end{array}
\end{equation}
$\mathcal{L}_{sse}^{q}(\cdot)$ denotes the sum of squared errors on query set $Q_k$, with learning rates $\alpha_3$, $\alpha_4$.

In the second stage, with $f_v$ and $f_w$ fixed, we update $f_\theta$ via bi-level optimization to explicitly encode transferability. We first fine-tune $f_\theta$ for each task using its support set:
\begin{equation}
\label{eq:f_k}
f_\theta^k \gets f_\theta - \beta_1 \nabla_{f_\theta} \mathcal{L}_{ssl}^{s}(B^{aug}_k, f_\theta),
\end{equation}
where $\beta_1$ is the learning rate. Then, we update $f_\theta$ by minimizing the total query loss across tasks:
\begin{equation}
\label{eq:f}
f_\theta \gets f_\theta - \beta_2 \nabla_{f_\theta} \sum_{k=1}^{K} \mathcal{L}_{sse}^{q}(B^{aug}_k, f_\theta^k),
\end{equation}
with $\beta_2$ as the learning rate.

This two-stage bi-level optimization mitigates task conflicts that hinder learning, and enables multi-task modeling in SSL to enhance representation transferability.

\begin{table*}[t]
    \centering
    \setlength{\tabcolsep}{3 pt}{
    \small
    \begin{tabular}{lccccccccc}
    \toprule
    \multirow{2.5}{*}{Method} &\multicolumn{2}{c}{SOT} &\multicolumn{1}{c}{VOS}&\multicolumn{2}{c}{MOT}&\multicolumn{2}{c}{MOTS} &\multicolumn{1}{c}{PoseTrack}\\
    \cmidrule(lr){2-3} \cmidrule(lr){4-4} \cmidrule(lr){5-6} \cmidrule(lr){7-8} \cmidrule(lr){9-9}
    ~ & AUC$_{\rm{XCorr}}$ & AUC$_{\rm{DCF}}$ & J-mean & IDF1 & HOTA & IDF1 & HOTA & IDF1 \\
    \midrule
    SimCLR & 47.3 / 51.9 & 61.3 / 50.7 & 60.5 / 56.5 & 66.9 / 75.6 & 57.7 / 63.2 & 65.8 / 67.6 & 67.7 / 69.5 & 72.3 / 73.5  \\
    MoCo &  50.9 / 47.9 & 62.2 / 53.7 & 61.5 / 57.9 & 69.2 / 74.1 & 59.4 / 61.9 & 70.6 / 69.3 & 71.6 / 70.9 & 72.8 / 73.9 \\
    SwAV & 49.2 / 52.4 & 61.5 / 59.4 & 59.4 / 57.0 & 65.6 / 74.4 & 56.9 / 62.3 & 68.8 / 67.0 & 69.9 /69.5 & 72.7 / 73.6 \\
    BYOL &  48.3 / 55.5 & 58.9 / 56.8 & 58.8 / 54.3 & 65.3 / 74.9 & 56.8 / 62.9 & 70.1 / 66.8 & 70.8 / 69.3 & 72.4 / 73.8 \\
    Barlow Twins &  44.5 / 55.5 &  60.5 / \textbf{60.1} &  61.7 / 57.8 &  63.7 / 74.5 &  55.4 / 62.4 &  68.7 / 67.4 &  69.5 / 69.8 &  72.3 / 74.3 \\
    \midrule
    SimCLR+TC$^2$ & {51.3} / 54.2 & \textbf{63.8} / 54.2 & \textbf{63.4} / \textbf{59.3} & \textbf{70.8} / \textbf{78.3} & \textbf{62.9} / \textbf{64.5} & 67.4 / \textbf{70.5} & 68.9 / 71.1 & 73.2 / 74.3 \\
    BYOL+TC$^2$ &  \textbf{51.6} / \textbf{58.1} & 59.9 / 58.9 & 61.7 / 56.9 & 67.3 / 76.6 & 57.9 / 63.8 & \textbf{72.8} / 68.1 & \textbf{73.8} / \textbf{72.3} & \textbf{75.6} / \textbf{76.2} \\
    \bottomrule
    \end{tabular}
    }
    \caption{Transfer learning on video tracking tasks. All methods use the ResNet-50 backbone with the best results in bold.}
    \label{tab:trans_video}
\end{table*}

\section{Experiments}
\label{Exp}

In this section, we conduct extensive experiments on benchmark datasets to evaluate the effectiveness of the proposed TC$^2$. More details and results are provided in the Appendix.

\subsection{Transfer Learning}
\label{Sec:Transfer}
In this section, we study the transferability of our proposed method under different transfer tasks, covering 7 benchmark datasets. Note that we also validate the effectiveness of our proposed method for unsupervised learning, semi-supervised learning and generative tasks in the Appendix.

\textbf{Video-based Task.} We evaluate TC$^2$ on video tasks by transferring the pre-trained model to the UniTrack benchmark~\cite{wang2021different}. Table~\ref{tab:trans_video} reports the results on five tasks using features from [layer3/layer4] of ResNet-50. The results show that our proposed method has improvements compared with existing self-supervised methods. SimCLR increase about 3\% in the VOS~\cite{perazzi2016benchmark}, and BYOL increase over 2\% in the MOT~\cite{milan2016mot16}.

\textbf{Object Detection and Instance Segmentation.} We evaluate our method on VOC 07 ~\cite{everingham2010pascal} and COCO~\cite{lin2014microsoft} for object detection and instance segmentation tasks with the most commonly used protocol~\cite{zbontar2021barlow}. Several different self-supervised methods are used for comparison.
We report the comparison results before and after introducing our proposed method in Table~\ref{tab:trans_odis}, showing that TC$^2$ brings stable performance improvements on different transfer tasks. For example, BYOL+TC$^2$ achieves SOTA on the VOC 07 detection task, and VICReg+TC$^2$ achieves SOTA on the COCO instance segmentation task.

\begin{table}
\centering
    \setlength{\tabcolsep}{0.6 pt}{
    \small
    \begin{tabular}{lccccccccc}
    \toprule
    \multirow{2.5}{*}{Method} &\multicolumn{3}{c}{VOC 07 det} &\multicolumn{3}{c}{COCO det}&\multicolumn{3}{c}{COCO seg}\\
    \cmidrule(lr){2-4} \cmidrule(lr){5-7} \cmidrule(lr){8-10}  
	& AP$_{50}$ & AP & AP$_{75}$ & AP$_{50}$ & AP & AP$_{75}$ & AP$^{m}_{50}$ & AP$^{m}$ & AP$^{m}_{75}$\\
    \midrule
     Supervised & 74.4  & 42.4 & 42.7 & 58.2 & 38.2 & 41.2 & 54.7 & 33.3 & 35.2\\
    \midrule
    SimCLR & 75.9 & 46.8 & 50.1 & 57.7 & 37.9 & 40.9 & 54.6 & 33.3 & 35.3\\
    MoCo & 77.1 & 46.8 & 52.5 & 58.9 & 39.3 & 42.5 & 55.8 & 34.4 & 36.5\\
    BYOL & 77.1 & 47.0 & 49.9 & 57.8 & 37.9 & 40.9 & 54.3 & 33.2 & 35.0\\
    MAE & 77.6 & 48.5 & 51.2 & 59.8 & 38.2 & 41.6 & 56.7 & 33.9 & 37.1 \\ 
    SimSiam & 77.3 & 48.5 & 52.5 & 59.3 & 39.2 & 42.1 & 56.0 & 34.4 & 36.7\\
    Barlow Twins & 75.7 & 47.2 & 50.3  & 59.0 & 39.2 & 42.5 & 56.0 & 34.3 & 36.5\\
    SwAV & 75.5 & 46.5 & 49.6 & 58.6 & 38.4 & 41.3 & 55.2 & 33.8 & 35.9\\
    VICRegL & 75.9 & 47.4 & 52.3 & 59.2 & 39.8 & 42.1 & 56.5 & 35.1 & 36.8\\    
     \midrule
    SimCLR+TC$^2$  & 78.3 & 49.3 & 52.3 & 59.6 & 40.7 & 43.7 & 56.8 & 36.5 & 37.2 \\
    BYOL+TC$^2$  & 79.7 & \textbf{50.7} & 52.5 & 59.7 & 39.9 & 44.1 & 56.3 & 35.3 & 38.3 \\
    MAE+TC$^2$ & \textbf{81.2} & 50.6 & 53.4 & 62.3 & 40.5 & 44.7 & \textbf{59.2} & 35.1 & \textbf{40.0} \\ 
    VICRegL+TC$^2$  & 78.2 & 49.5 & \textbf{54.1} & \textbf{62.7} & \textbf{42.6} & \textbf{45.3} & 58.9 & \textbf{37.9} & 39.7 \\
    \bottomrule
\end{tabular}}
\caption{Transfer learning on objection detection and instance segmentation. ``AP'' is the average precision, ``$\text{AP}_{N}$'' represents the average precision when the IoU (Intersection and Union Ratio) threshold is $N\%$.}
\label{tab:trans_odis}
\end{table}

\textbf{OOD Tasks.} We also assess our method on benchmark datasets that address the OOD challenge, i.e., PACS \cite{xu2019self}, ColoredMNIST \cite{nam2020learning}, and OfficeHome \cite{venkateswara2017deep}. 
Specifically, we assess the performance of SSL baselines both before and after incorporating the proposed TC$^2$ across three datasets. For PACS, we train the SSL models on these domains (Photo, Sketch, and Cartoon) and test them on all four domains (Photo, Art, Cartoon, and Sketch), with the average performance also reported. For ColoredMNIST, we adopt the experimental setup from \cite{gat2020removing}, evaluating the ability to generalize to new classes after training on base classes. Lastly, for OfficeHome, we randomly select one domain as the source for training and another as the target domain. Labels for the source domain are predefined, whereas those for the target domain remain unknown. We then compare the performance of SimCLR with and without the introduction of TC$^2$.
As shown in Table \ref{tab:app_ood_pacs}, Table \ref{tab:app_ood_minst} and Table \ref{tab:app_ood_office}, TC$^2$ significantly enhances performance, thereby confirming its effectiveness on OOD tasks.

\begin{table}[t]
\centering
\setlength{\tabcolsep}{3 pt}{
\small
\begin{tabular}{l|c|c|c|c|c}
\toprule
Method   & Photo & Sketch & Cartoon & Art & Average \\
\hline
SimCLR      & 86.4 & 85.1 & 87.2  & 74.3 & 83.3    \\
SimCLR+TC$^2$ & 88.8 & 89.2 & 88.5 & 78.3 & 86.2 \\
BYOL        & 83.9  & 84.6   & 82.7 & 64.5 & 78.9    \\
BYOL+TC$^2$   & 85.6 & 87.2 & 85.0 & 69.9 & 81.9 \\
Barlow Twins & 83.1 & 82.7 & 83.4 & 64.8 & 78.5 \\
Barlow Twins+TC$^2$ & 86.5 & 84.6 & 86.7 & 67.4 & 81.3 \\
MAE & 86.9 & 87.2 & 87.6 & 75.4 & 84.3 \\
MAE+TC$^2$ & 89.1 & 90.0 & 89.2 & 78.9 & 86.8 \\
SwAV & 87.2 & 85.2 & 87.6 & 74.9 & 83.7 \\
SwAV+TC$^2$ & 89.9 & 88.6 & 89.1 & 78.2 & 86.5 \\
VICRegL & 88.0 & 86.1 & 87.9 & 75.2 & 84.3 \\
VICRegL+TC$^2$ & 91.5 & 89.3 & 90.2 & 79.6 & 87.7 \\
\bottomrule
\end{tabular}}
\caption{Transfer learning on PACS dataset.}
\label{tab:app_ood_pacs}
\end{table}

\begin{table}
\centering
\setlength{\tabcolsep}{2 pt}{
\small
\begin{tabular}{lcc}
\toprule
Method & Average Accuracy(\%) & Worst Accuracy (\%) \\
\toprule
SimCLR      & 85.2 & 82.4\\
SimCLR+TC$^2$ & 88.9 & 87.0 \\
BYOL      & 85.9 & 83.1 \\
BYOL+TC$^2$ & 89.4 & 87.8 \\
MAE & 86.8 & 84.0 \\
MAE+TC$^2$ & 90.1 & 88.3 \\
\bottomrule
\end{tabular}}
\caption{Transfer learning on ColoredMNIST dataset.}
\label{tab:app_ood_minst}
\end{table}

\subsection{Ablation Studies}
\label{Sec:Ablation}
In this subsection, we conduct ablation studies to analyze how our proposed method perform well.

\textbf{Loss function.} Our loss function contains two important components: $\mathcal{L}_v$, $\mathcal{L}_w$. We reduce the correlation between generative factors by minimizing $\mathcal{L}_v$ so that they can represent different semantic information. We also enforce correspondence between tasks and their generative factors by minimizing $\mathcal{L}_w$, while constraining the sparsity and diversity of the factors. Experimental results demonstrate the necessity of each component, as showned in Table~\ref{tab:abl_loss}. 
% Table~\ref{tab:abl_loss} shows the impact of different loss functions. 

\textbf{Bi-level mechanism}. Here, we conduct experiments to explore the impact of bi-level mechanism, as shown in Table~\ref{tab:abl_loss}. Experimental results show that the bi-level optimization mechanism indeed improves the performance of the model, demonstrating its effectiveness.

\textbf{The number of tasks $n_t$.}
We divide each batch equally into multiple parts to construct SSL tasks. Here, we conduct experiments to analyze the impact of the number of tasks $n_t$ on the experimental results on the CIFAR10 dataset. In the experiment, we set the number of tasks to 2, 4, 8, 16, and 32. Table~\ref{tab:abl_task} shows the effect of the number of tasks on the experiments. The experimental results show that $n_t=4$ achieves the best results, also our settings.

\section{Conclusion}

This paper aims to explore how to improve the transferability of SSL.
Through both theoretical analysis and empirical studies, we find that incorporating task-level information into SSL improves its transferability. However, task conflict introduces instability during the training process, limiting the transferability. To address this issue, we propose TC$^2$ to alleviate the negative impact of task conflict in self-supervised learning. TC$^2$ incorporates task-level information, calibrates task–factor correspondence through dedicated extraction functions, and optimizes the model using a two-stage bi-level optimization mechanism. Experimental results on various downstream tasks confirm the effectiveness of TC$^2$ in improving the transferability of SSL.

\section*{Acknowledgements}
The authors would like to thank the anonymous reviewers for their valuable comments. This work was supported by the National Natural Science Foundation of China under Grants 62506355 and 42506186, and the numerical calculations in this study were partially performed on the ORISE Supercomputer (DFZX202416).

\bibliography{TC2_Final}

\newpage

\appendix
\section*{Appendix}
The appendices provide supplementary materials, including theoretical proofs, implementation details, and extended experimental results, as outlined below:\\
\textbf{Appendix I} presents the proof of the theorem 1.\\
\textbf{Appendix II} describes the setup and all results of the motivation experiment.\\
\textbf{Appendix III} provides the pseudo-code of the proposed method.\\
\textbf{Appendix IV} details the experimental setup and full results, including unsupervised validation, semi-supervised validation, transfer learning, and ablation studies.

\section{I: Proof of Theorem \ref{thm1}}
\label{App_Thm}
Taking contrastive learning as an example, SSL can be regarded as multiple binary classification problems, whose label information is generated by the adopted data augmentation. Therefore, we consider using two binary classification problems as research objects to analyze the task conflict problem caused by the introduction of multi-task information into SSL. Here, $Y_i$ and $Y_j$ represent the label variables of task $i$ and task $j$ respectively, and belong to $\left\{ { \pm 1} \right\}$, while $X_i$ and $X_j$ correspond to sample variables.

From a generation perspective, training samples can be directly determined by generative factors. We set the generative factors of samples in tasks $i$ and $j$ to be $F^u_i$ and $F^u_j$ respectively, which can be characterized by Gaussian distribution. Assuming that there is no overlap between $F^u_i$ and $F^u_j$, they can be represented by a multidimensional Gaussian distribution:
\begin{equation}
\begin{array}{l}
    F^u_i\sim \mathcal{N}(Y_i\cdot \mu_i,\sigma_i^2 I)\\
    F^u_j\sim \mathcal{N}(Y_j\cdot \mu_j,\sigma_j^2 I)
\end{array}
\end{equation}
where $\mu_i, \sigma_i^2$ represents the mean and variance respectively. For simplicity, we assume that labels originate from different distributions and the sampling distribution of labels is the same, $P(Y=1)=P(Y=-1)=0.5$ (not mandatory). In the multi-task joint learning setting, we define $c_{ij}$ as the task correlation caused by task conflict, $P(Y_i=Y_j) = 1-P(Y_i=Y_j) =c_{ij}$. Next, we can deduce that the optimal classifier for each task has non-zero weights for its non-generative factors according to the following theorem.

\renewcommand{\thetheorem}{1}
\begin{theorem}
If the correlation between $Y_i$ and $Y_j$ is not equal to 0.5, then the optimal model for task $i$ has a non-zero weight on $F^u_j$. If the correlation is equal to 0.5 with limited training samples, then the optimal classifier for task $i$ also has non-zero weight on factor $F^u_j$.
\end{theorem}

\begin{proof}
We train a model using joint learning of two tasks and analyze whether the best classifier of the target task for a single task will be affected by the generative factors from other tasks. Taking task $i$ as an example, the optimal Bayesian classifier is expressed as:
\begin{equation}
\begin{aligned}
    P(Y_i|F^u_i, F^u_j) &=\frac{P(Y_i,F^u_i, F^u_j)}{P(F^u_i, F^u_j)} \\
    &=\frac{P(Y_i,F^u_i, F^u_j)}{ {\textstyle \sum\nolimits_{Y_i\in \left \{ -1,1 \right \}  }} P(Y_i,F^u_i, F^u_j)} 
\end{aligned}
\label{eq:classifier_1}
\end{equation}
where the probability of $P(Y_i,F^u_i, F^u_j)$ could be written as:
\begin{equation}
\begin{aligned}
     P(Y_i,F^u_i, F^u_j)
     &= P(Y_i,F^u_i)\cdot P(F^u_j|Y_i,F^u_i)\\
     &= P(Y_i,F^u_i)\cdot P(F^u_i|Y_i)\\
     &= P(Y_i,F^u_i)\cdot\sum\nolimits_{Y_j\in \{ -1,1 \} }P(F^u_j,Y_j|Y_i)\\
     &= P(Y_i)P(F^u_i|Y_i)\cdot\\
     & \sum\nolimits_{Y_j\in \{ -1,1 \} }P(F^u_j|Y_j)P(Y_j|Y_i)
\end{aligned}
\end{equation}
Since the generative factors obey Gaussian distribution, $P(Y_{i},F^u_i,F^u_j)=sigmoid(\frac{\mu _i}{\sigma^2_i}F^u_i+\frac{\mu _j}{\sigma^2_j}F^u_j)$ where $\frac{\mu _i}{\sigma^2_i}$ and $\frac{\mu _j}{\sigma^2_j}$ are the regression vectors for the optimal Bayesian classifier, then we have:
\begin{equation}
\begin{aligned}
     &P(Y_i,F^u_i,F^u_j)\\ &= P(Y_i)P(F^u_i|Y_i)\cdot\sum\nolimits^{}_{Y_j\in \{ -1,1 \} }P(F^u_j|Y_j)P(Y_j|Y_i)\\
     &\propto e^{Y_i\cdot \frac{\mu _i}{\sigma^2_i}F^u_i}(c_{ij}e^{Y_i\cdot \frac{\mu _j}{\sigma^2_j}F^u_j}+(1-c_{ij})e^{-Y_i\cdot \frac{\mu _j}{\sigma^2_j}F^u_j}) \\[8pt]
     &=c_{ij}e^{Y_i\cdot (\frac{\mu _i}{\sigma^2_i}F^u_i+\frac{\mu _j}{\sigma^2_j}F^u_j)}+(1-c_{ij})e^{Y_i\cdot (\frac{\mu _i}{\sigma^2_i}F^u_i-\frac{\mu _j}{\sigma^2_j}F^u_j)}\\
     &=c_{ij}e^{Y_i\cdot \gamma_1}+(1-c_{ij})e^{Y_i\cdot \gamma_2}
\end{aligned}
\end{equation}
where $\gamma_1=\frac{\mu _i}{\sigma^2_i}F^u_i+\frac{\mu _j}{\sigma^2_j}F^u_j,\gamma_2=\frac{\mu _i}{\sigma^2_i}F^u_i-\frac{\mu _j}{\sigma^2_j}F^u_j$, Thus, we have:
\begin{equation}
\begin{aligned}
    P(Y_i|F^u_i,F^u_j) = \frac{1}{1+\frac{p_{tc}e^{Y_i\cdot \gamma_1}+(1-p_{tc})e^{Y_i\cdot \gamma_2 }}{p_{tc}e^{-Y_i\cdot \gamma_1}+(1-p_{tc})e^{-Y_i\cdot \gamma_2 }} } 
\end{aligned}
\label{eq:classifier_2}
\end{equation}

When $c_{ij}=0.5$:
\begin{equation}
\begin{aligned}
    P(Y_i|F^u_i,F^u_j) = \frac{1}{1+e^{Y_i\cdot (\gamma_1 +\beta^-)}} = \frac{1}{1+e^{2Y_i\cdot (\frac{\mu _i}{\sigma^2_i}F^u_i)}}
\end{aligned}
\end{equation}
It can be shown that when the correlation is 0.5, the optimal classifier only uses the generative factors $F^u_{i}$ for task $i$.

When $c_{ij}\ne 0.5$, according to Eq.\ref{eq:classifier_2}, the optimal classifier for each task will be affected by the generative factors of another task.

\end{proof}

%%%%%%%%%%% Algorithm %%%%%%%%%%%
\begin{algorithm}[t]
\caption{SSL with TC$^2$}
\label{alg:algorithm}
\begin{flushleft}
\textbf{Input}: Training set $X$; initialize self-supervised model with a encoder $f_\theta$; Randomly initialize networks $f_v$ and $f_w$.\\
\textbf{Parameter}: Learning rates $\beta_1$ and $\beta_2$ for the learning of $f_{\theta}$; Learning rates $\alpha_1$, $\alpha_2$, $\alpha_3$, $\alpha_4$ for the learning of $W$ and $V$; Parameter $\lambda _s$ for the balancing term in loss function.\\
\textbf{Output}: self-supervised model $f_\theta$\\
\end{flushleft}
\begin{algorithmic}[1]
\FOR{sample batch $X_{tr}$ from $X$}
        \STATE Obtain $B =\{ B_1,...,B_k,...,B_K\}$ via splitting
        \STATE Obtain $B^{aug} =\{ B_1^{aug},...,B_k^{aug},...,B_K^{aug}\}$ via augmentation
        \STATE Update $f_v$ and $f_w$ with Eq.\ref{eq:wv_inner} and Eq.\ref{eq:wv_outer}
        \STATE Update SSL model $f_\theta$ via Eq.\ref{eq:f_k} and Eq.\ref{eq:f}

\ENDFOR
\STATE \textbf{return} $f_\theta$
\end{algorithmic}
\end{algorithm}

\section{II: Motivational Experiment}
\label{App_toy}
In this section, we introduce the full details and experimental results of the motivating experiments in the main text.

\begin{figure}[t]
    \centering
    \includegraphics[width=\linewidth]{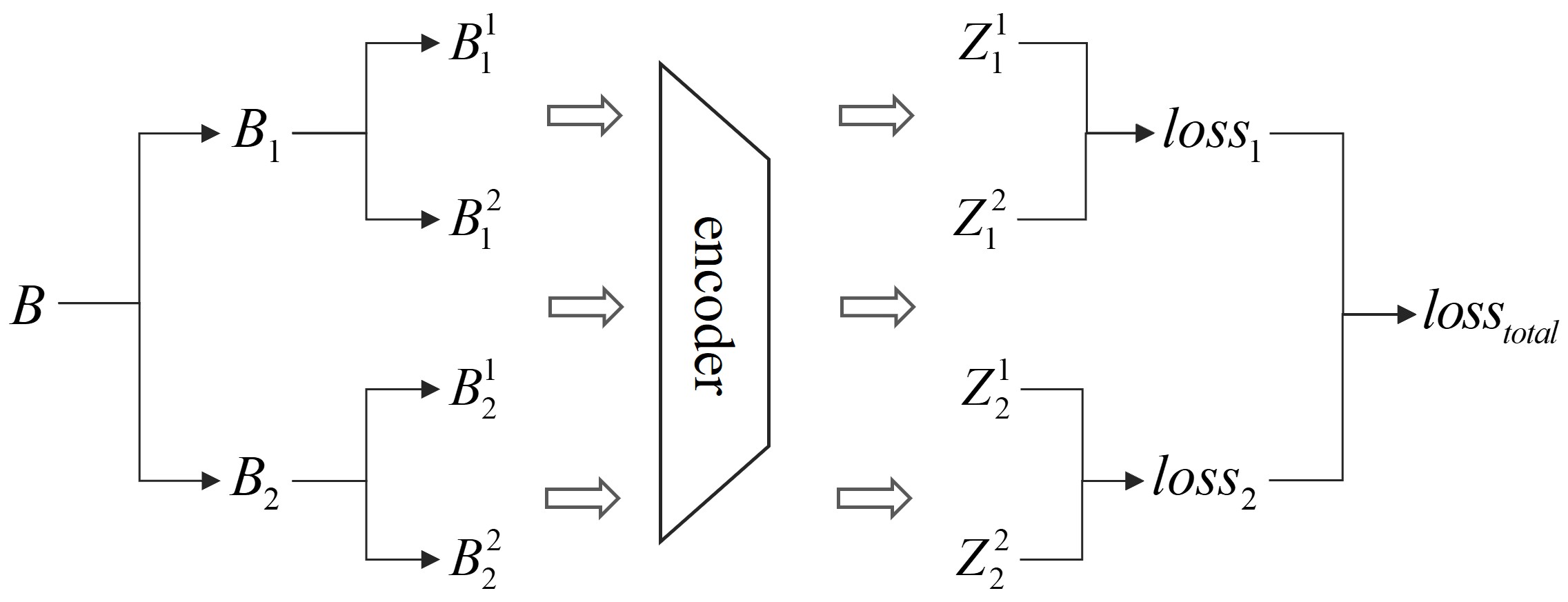}
    \caption{Model training pipeline for the motivational experiment. The encoder is jointly optimized based on two tasks.}
    \label{fig:motiva_pipe}
\end{figure}

\begin{figure*}[t]
    \centering
    \begin{subfigure}[b]{0.33\textwidth}
         \includegraphics[width=\textwidth]{motivation_1_bar.pdf}
         \caption{Performance on CIFAR-100}
         % \label{fig:motivation_comparison}
    \end{subfigure}
    \hfill
    \begin{subfigure}[b]{0.33\textwidth}
         \includegraphics[width=\textwidth]{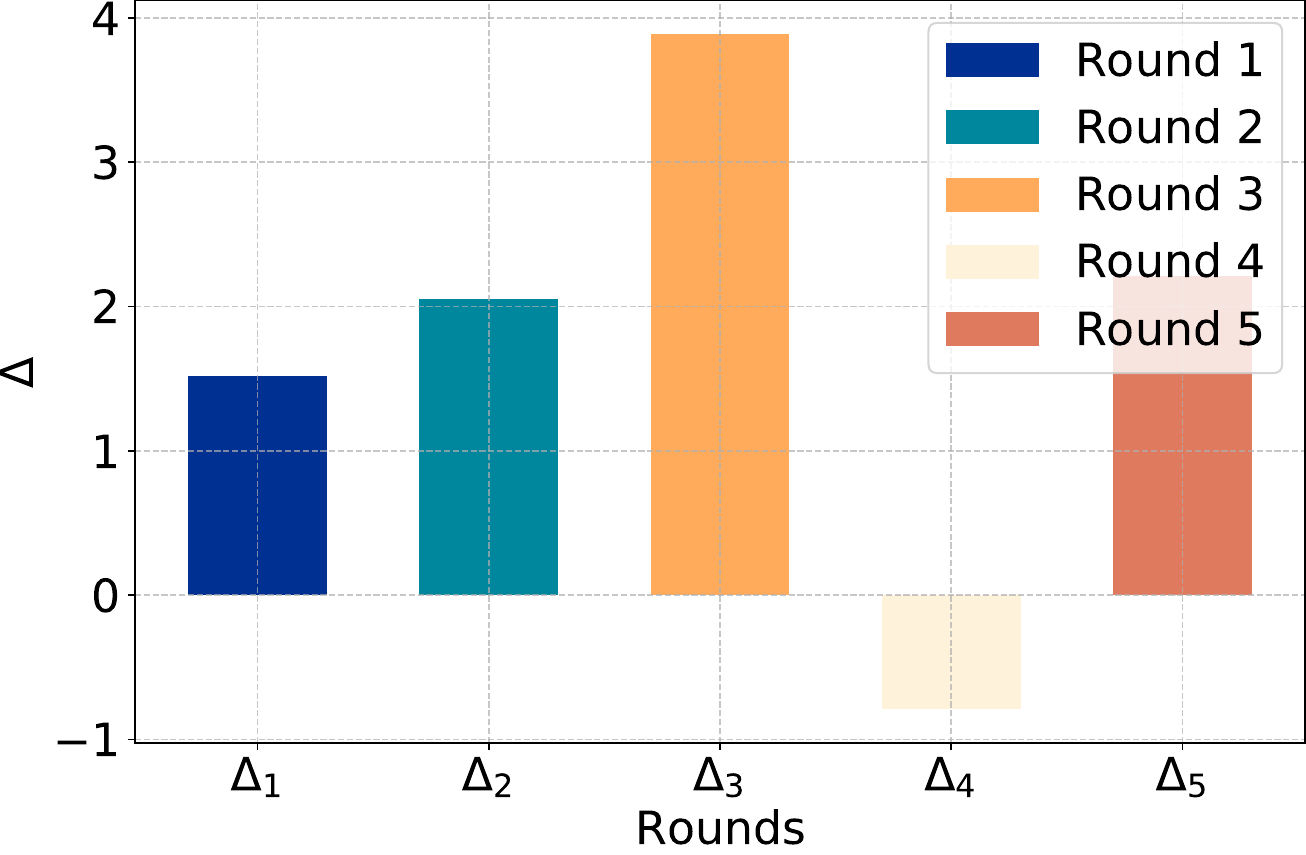}
         \caption{Performance on Flower102}
         % \label{fig:motivation_comparison}
    \end{subfigure}
    \hfill
    \begin{subfigure}[b]{0.33\textwidth}
         \includegraphics[width=\textwidth]{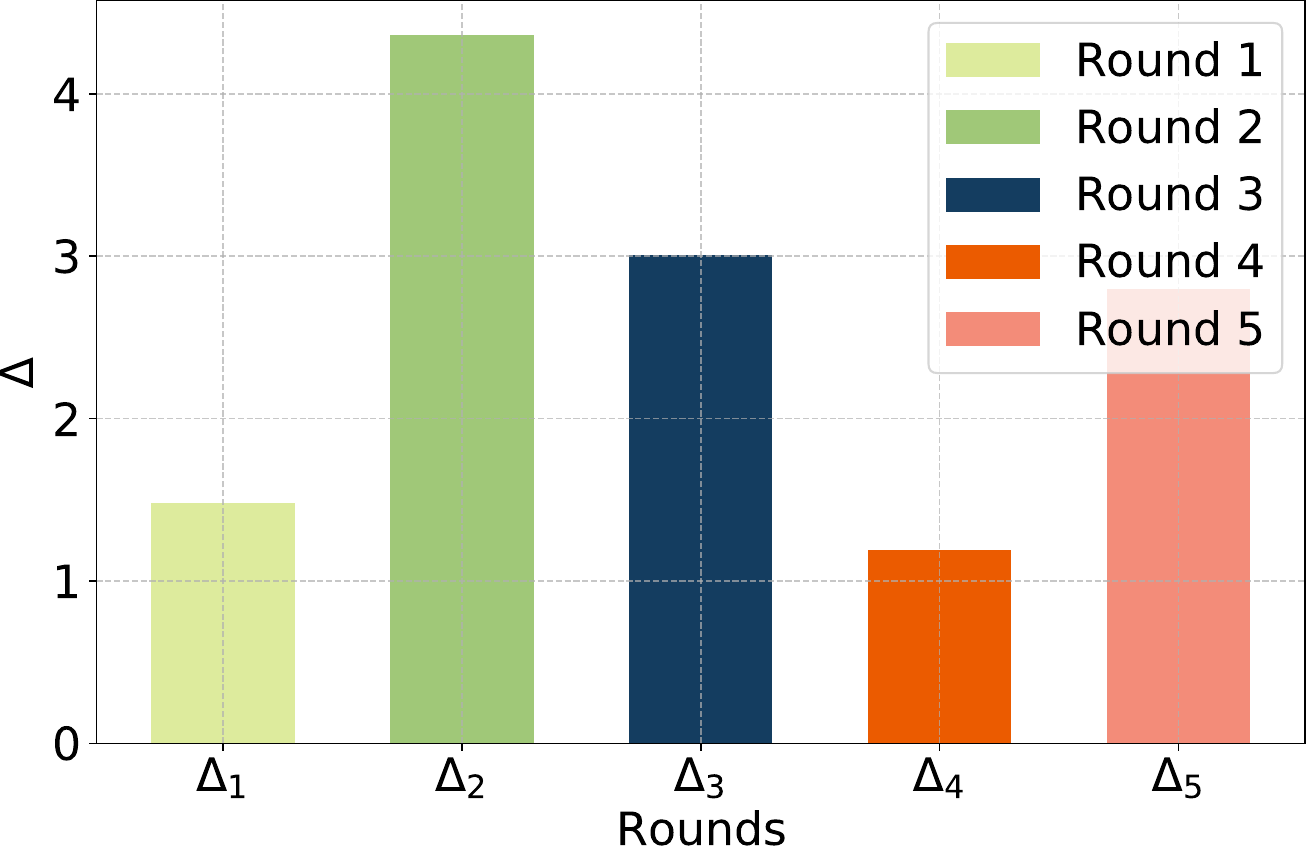}
         \caption{Performance on Aircraft}
         % \label{fig:motivation_comparison}
    \end{subfigure}
    \caption{Effect of task-level information. It shows the effect of SimCLR trained on ImageNet and tested on three different datasets before and after introducing task-level information in five runs.}
    \label{fig_app:motivation}
\end{figure*}

\begin{table*}[t]
\centering
\setlength{\tabcolsep}{3pt}
\small
\begin{tabular}{l|cc|cc|cc}
\toprule
\multirow{2}{*}{Method} & \multicolumn{2}{c|}{CIFAR-10} & \multicolumn{2}{c|}{CIFAR-100} & \multicolumn{2}{c}{STL-10} \\
~ & Top-1 & 5-NN & Top-1 & 5-NN & Top-1 & 5-NN \\
\midrule
SimCLR  &  91.80$\pm$0.16  &  88.42$\pm$0.15  &  66.83$\pm$0.23  &  56.57$\pm$0.18  &  90.51$\pm$0.14  &  85.68$\pm$0.11 \\
BYOL  &  91.73$\pm$0.23  &  89.26$\pm$0.22  &  66.60$\pm$0.16  &  56.82$\pm$0.18  &  91.99$\pm$0.13  &  88.64$\pm$0.20 \\
MoCo  &  91.69$\pm$0.13  &  88.66$\pm$0.15  &  67.02$\pm$0.16  &  56.29$\pm$0.25  &  90.64$\pm$0.29  &  88.01$\pm$0.20 \\
Barlow Twins  &  90.88$\pm$0.19  &  88.78$\pm$0.21  &  66.67$\pm$0.10  &  56.39$\pm$0.24  &  90.71$\pm$0.13  &  85.31$\pm$0.23 \\
SimSiam  &  91.71$\pm$0.27  &  88.65$\pm$0.17  &  67.22$\pm$0.26  &  56.36$\pm$0.19  &  91.01$\pm$0.20  &  88.16$\pm$0.22 \\
W-MSE  &  91.99$\pm$0.12  &  89.87$\pm$0.25  &  67.64$\pm$0.16  &  56.45$\pm$0.26  &  91.75$\pm$0.23  &  88.59$\pm$0.15 \\
SwAV  &  90.17$\pm$0.19  &  89.52$\pm$0.24  &  66.56$\pm$0.17  &  57.01$\pm$0.24  &  90.72$\pm$0.29  &  86.24$\pm$0.26 \\
DINO  &  91.83$\pm$0.25  &  90.15$\pm$0.33  &  67.15$\pm$0.21  &  56.48$\pm$0.19  &  91.03$\pm$0.12  &  86.15$\pm$0.25 \\
SSL-HSIC  &  91.95$\pm$0.14  &  89.91$\pm$0.17  &  67.22$\pm$0.26  &  57.01$\pm$0.27  &  92.09$\pm$0.20  &  88.91$\pm$0.29 \\
C-JEPA  &  91.65$\pm$0.08  &  90.05$\pm$0.21  &  67.46$\pm$0.13  &  56.88$\pm$0.22  &  91.89$\pm$0.13  &  88.07$\pm$0.15 \\
VICRegL  &  90.99$\pm$0.13  &  88.75$\pm$0.26  &  68.03$\pm$0.32  &  57.34$\pm$0.29  &  92.12$\pm$0.26  &  90.01$\pm$0.21 \\
\midrule
SimCLR+TC$^2$  &  93.84$\pm$0.22  &  91.75$\pm$0.23  &  \textbf{69.89$\pm$0.26}  &  59.65$\pm$0.27  &  93.89$\pm$0.27  &  88.99$\pm$0.26 \\
BYOL+TC$^2$  &  \textbf{94.89$\pm$0.24}  &  \textbf{91.96$\pm$0.22}  &  68.93$\pm$0.25  &  59.79$\pm$0.26  &  \textbf{94.93$\pm$0.25}  &  \textbf{91.19$\pm$0.27} \\
C-JEPA+TC$^2$  &  93.55$\pm$0.12  &  91.34$\pm$0.17  &  68.23$\pm$0.11  &  58.52$\pm$0.12  &  93.47$\pm$0.15  &  89.17$\pm$0.23 \\
Barlow Twins+TC$^2$  &  94.41$\pm$0.27  &  91.81$\pm$0.26  &  68.86$\pm$0.21  &  \textbf{59.99$\pm$0.28}  &  94.35$\pm$0.21  &  89.51$\pm$0.27 \\
\bottomrule

\multirow{2}{*}{Method} & \multicolumn{2}{c|}{Tiny ImageNet} & \multicolumn{2}{c|}{ImageNet-100} & \multicolumn{2}{c}{ImageNet} \\
~ & Top-1 & 5-NN & Top-1 & 5-NN & Top-1 & 5-NN \\
\midrule
SimCLR  &  48.82$\pm$0.15  &  32.86$\pm$0.25  &  70.15$\pm$0.16  &  89.75$\pm$0.14  &  68.32$\pm$0.31  &  89.76$\pm$0.23 \\
BYOL  &  51.01$\pm$0.12  &  36.24$\pm$0.28  &  74.89$\pm$0.23  &  92.83$\pm$0.21  &  74.31$\pm$0.29  &  91.62$\pm$0.25 \\
MoCo  &  50.92$\pm$0.22  &  35.55$\pm$0.18  &  72.81$\pm$0.12  &  89.75$\pm$0.11  &  71.37$\pm$0.27  &  88.94$\pm$0.11 \\
Barlow Twins  &  49.74$\pm$0.26  &  33.61$\pm$0.19  &  72.88$\pm$0.23  &  90.99$\pm$0.19  &  73.22$\pm$0.31  &  91.01$\pm$0.27 \\
SimSiam  &  51.14$\pm$0.21  &  35.67$\pm$0.18  &  73.01$\pm$0.22  &  92.61$\pm$0.27  &  70.02$\pm$0.14  &  88.76$\pm$0.23 \\
W-MSE  &  49.22$\pm$0.16  &  35.44$\pm$0.10  &  76.01$\pm$0.27  &  93.12$\pm$0.21  &  70.85$\pm$0.31  &  91.57$\pm$0.20 \\
SwAV  &  52.02$\pm$0.25  &  37.40$\pm$0.11  &  75.77$\pm$0.18  &  92.86$\pm$0.17  &  69.12$\pm$0.24  &  89.38$\pm$0.26 \\
DINO  &  51.13$\pm$0.30  &  37.86$\pm$0.19  &  75.43$\pm$0.18  &  93.32$\pm$0.19  &  70.58$\pm$0.24  &  91.32$\pm$0.27 \\
SSL-HSIC  &  51.37$\pm$0.15  &  36.03$\pm$0.12  &  74.77$\pm$0.22  &  92.56$\pm$0.25  &  72.13$\pm$0.28  &  90.33$\pm$0.29 \\
VICRegL  &  51.52$\pm$0.13  &  36.24$\pm$0.16  &  75.96$\pm$0.21  &  92.97$\pm$0.26  &  70.24$\pm$0.27  &  91.60$\pm$0.24 \\
\midrule
SimCLR+TC$^2$  &  52.21$\pm$0.25  &  37.42$\pm$0.26  &  74.89$\pm$0.28  &  92.97$\pm$0.29  &  72.99$\pm$0.30  &  92.96$\pm$0.27 \\
BYOL+TC$^2$  &  \textbf{54.25$\pm$0.28}  &  \textbf{39.45$\pm$0.27}  &  \textbf{77.74$\pm$0.29}  & \textbf{ 95.82$\pm$0.25}  &  \textbf{77.05$\pm$0.26}  &  \textbf{93.99$\pm$0.29} \\
Barlow Twins+TC$^2$  &  52.94$\pm$0.19  &  37.55$\pm$0.23  &  75.95$\pm$0.22  &  93.34$\pm$0.23  &  75.94$\pm$0.24  &  92.99$\pm$0.27 \\
\bottomrule
\end{tabular}
\caption{Unsupervised Evaluation. The classification accuracies with a ResNet-18 for the first four datasets and ResNet-50 for the last two datasets. The best results are highlighted in bold.}
\label{tab:self_un1}
\end{table*}

Existing SSL models do not explicitly incorporate transferability into their loss functions or architectures, prompting us to explore whether modeling transferability during training can further enhance performance. Motivated by meta-learning~\cite{finn2017model, hospedales2021meta} which acquires transferable knowledge from diverse tasks, we construct multiple tasks during the SSL training phase to examine the impact on transferability.
In our experiments, each randomly sampled training batch \(B\) is evenly divided into two subsets, \(B_1\) and \(B_2\), which serve as inputs for two independent tasks. For each task, samples are augmented using two random transformation strategies to generate paired views. These views are processed by a shared encoder (ResNet-18) to extract feature representations. Each task follows a conventional contrastive learning framework (e.g., SimCLR), and the encoder is jointly optimized using the combined loss from both tasks, as illustrated in Figure \ref{fig:motiva_pipe}. We adopt SimCLR as the baseline model, training on ImageNet-100 and testing on CIFAR-100. First, we measure the baseline model's Top-1 accuracy without task construction, denoted as \( acc(\mathrm{SimCLR}) \). Then, under identical experimental settings, we evaluate the performance of the multi-task learning mechanism across five independent training rounds, with the \( i \)-th round's accuracy denoted as \( acc(\mathrm{SimCLR+T}, i) \). Finally, we calculate the improvement in classification performance for each round as \( \Delta_i = acc(\mathrm{SimCLR+T}, i) - acc(\mathrm{SimCLR}) \). The experimental results in Figure \ref{fig:motivation_comparison} show that Incorporating task-level information into SSL improves transferability, yet it remains vulnerable to gradient conflicts between tasks.

To validate the universality of this phenomenon, we further test on additional datasets. Specifically, in addition to CIFAR-100, we selected the Flower102~\cite{cibuk2019efficient} and Aircraft~\cite{maji2013fine} datasets to evaluate model transferability. The experimental results in Figure \ref{fig_app:motivation} reached a similar conclusion: incorporating task-level information into SSL significantly improves performance, although outcomes vary between the best and worst cases.

\section{III: Pseudo-Code}
\label{Pseudo-Code}

In this section, the pseudocode of the algorithm is presented in Algorithm~\ref{alg:algorithm}.

\section{IV: Experiment Details and Results}
\label{App_EDR}

\begin{table}[t]
\centering
\setlength{\tabcolsep}{3 pt}{
\small
\begin{tabular}{c|cccccc}
    \toprule
    Method  & pre-train data & ViT-B & ViT-L & ViT-H & ViT-H$_{448}$\\
    \midrule
    DINO & IN1K & 82.8 & - & - & -\\
    MoCo & IN1K & 83.2 & 84.1 & - & - \\
    BEiT & IN1K+DALLE & 83.2 & 85.2 & - & - \\
    MAE & IN1K & 83.6 & 85.9 & 86.9 & 87.8\\
    \midrule
    MAE+TC$^2$ & IN1K & 86.2 & 87.3 & 88.2 & 89.0\\
    \bottomrule
\end{tabular}}
\caption{Comparisons with previous results on ImageNet-1K. The ViT models are B/16, L/16, H/14. The pre-training data is the ImageNet-1K training set (except the tokenizer in BEiT was pre-trained on 250M DALLE data~\cite{ramesh2021zero}). All results are on an image size of 224, except for ViT-H with an extra result of 448.}
\label{tab:gssl_1}
\end{table}

\begin{table*}[t]
\centering
    \setlength{\tabcolsep}{3 pt}{
    \small
\begin{tabular}{l|c|c|c|c|c|c|c|c|c|c|c|c}
\toprule
Method      & A $\to$ C & A $\to$ P & A $\to$ R & C $\to$ A & C $\to$ P & C $\to$ R & P $\to$ A & P $\to$ C & P $\to$ R & R $\to$ A & R $\to$ C & R $\to$ P \\
\hline
SimCLR & 58.2  & 63.5  & 69.8  & 78.9  & 69.7  & 66.8  & 63.4  & 52.3  & 58.4  & 56.1  & 72.9  & 71.0 \\
SimCLR+TC$^2$ & 60.1 & 65.7 & 71.2 & 79.0 & 73.6 & 68.2 & 66.1 & 56.4 & 60.5 & 58.0 & 59.4 & 74.9 \\
BYOL & 58.5 & 64.1 & 69.2 & 78.1 & 68.5 & 67.2 & 64.0 & 53.8 & 58.7 & 56.9 & 72.1 & 71.4 \\
BYOL+TC$^2$ & 60.2 & 66.4 & 72.4 & 80.0 & 72.4 & 70.2 & 67.3 & 56.4 & 60.2 & 58.8 & 75.9 & 73.2 \\
Barlow Twins & 59.3 & 63.8 & 70.2 & 79.5 & 70.2 & 67.6 & 64.3 & 54.0 & 58.1 & 57.0 & 71.8 & 69.4 \\
Barlow Twins+TC$^2$ & 61.2 & 65.7 & 72.9 & 81.3 & 73.8 & 70.2 & 66.1 & 57.3 & 60.8 & 60.1 & 74.6 & 71.5 \\
\bottomrule
\end{tabular}}
\caption{Transfer learning on OfficeHome dataset.}
\label{tab:app_ood_office}
\end{table*}

\begin{table*}
    \centering
    \setlength{\tabcolsep}{3 pt}{
    % \small
    \begin{tabular}{lccccc}
    \toprule
    \multirow{2}*{Method} & \multirow{2}*{Epochs}  & \multicolumn{2}{c}{1\%} & \multicolumn{2}{c}{10\%} \\\cline{3-6}
    ~ & ~ & Top-1 & Top-5 & Top-1 &  Top-5 \\
    \midrule
    MoCo & 200 &  43.8$\pm 0.2$  & 72.3$\pm 0.1$ & 61.9$\pm 0.1$  & 84.6$\pm 0.2$\\
    BYOL & 200 &  54.8$\pm 0.2$  & 78.8$\pm 0.1$ & 68.0$\pm 0.2$  & 88.5$\pm 0.2$\\
    \midrule
    SimCLR & 1000 &  48.3$\pm 0.2$  & 75.5$\pm 0.1$ & 65.6$\pm 0.1$  & 87.8$\pm 0.2$\\
    MoCo & 1000 &  52.3$\pm 0.1$  & 77.9$\pm 0.2$ & 68.4$\pm 0.1$  & 88.0$\pm 0.2$\\
    BYOL & 1000 & 53.2$\pm 0.2$  & 78.4$\pm 0.2$ & 68.8$\pm 0.2$ & 89.0$\pm 0.1$\\
    SimSiam & 1000 & 54.9$\pm 0.2$  & 79.5$\pm 0.2$ & 68.1$\pm 0.1$ & 89.0$\pm 0.3$\\
    SwAV & 1000 & 53.9$\pm 0.1$  & 78.5$\pm 0.2$ & 70.2$\pm 0.1$  & 89.9$\pm 0.2$\\
    SSL-HSIC & 1000 & 55.4$\pm 0.3$  & 80.1$\pm 0.2$ & 70.4$\pm 0.1$  & 90.1$\pm 0.1$\\
    BarlowTwins & 1000 & 55.0$\pm 0.1$  & 79.2$\pm 0.1$ & 69.7$\pm 0.2$  & 89.3$\pm 0.2$\\
    VICRegL & 1000 & 54.9$\pm 0.1$  & 79.6$\pm 0.2$ & 67.2$\pm 0.1$  & 89.4$\pm 0.2$\\
    \midrule
SimCLR+TC$^2$ & 1000 & 50.9$\pm 0.2$ &  77.4$\pm 0.1$ & 67.3$\pm 0.2$  & 89.9$\pm 0.3$\\
BYOL+TC$^2$ & 1000 & 55.8$\pm 0.3$  & 80.2$\pm 0.2$ & 70.5$\pm 0.2$  & 90.8$\pm 0.1$\\
Barlow Twins+TC$^2$ & 1000 & \textbf{57.9$\pm 0.2$}  & \textbf{81.6$\pm 0.3$} & \textbf{71.9$\pm 0.2$} & \textbf{92.2$\pm 0.2$}\\
\bottomrule
    \end{tabular}
    }
\caption{Semi-Supervised Evaluation on ImageNet. We finetune the pre-trained model using 1\% and 10\% training samples of ImageNet, and the Top-1 and Top-5 under linear evaluation are reported. The best results are highlighted in bold.}
\label{tab:semi}
\end{table*}

\subsection{Data Augmentation and Paramter setting}

We adopt standard image augmentation strategies, including random cropping, resizing, horizontal flipping, color jittering, grayscaling, and Gaussian blurring. Specifically, each image is randomly cropped to a size ranging from 20\% to 100\% of the original area, with an aspect ratio randomly selected between $3/4$ and $4/3$. For generating two augmented views, the following transformations are applied with specified probabilities: horizontal flipping with probability 0.5, color jittering with configuration $(0.4, 0.4, 0.4, 0.1)$ with probability 0.8, and conversion to grayscale with probability 0.1. For ImageNet-100 and ImageNet, we follow slightly stronger augmentation settings: cropping size from 8\% to 100\% of the original area, color jittering with configuration $(0.8, 0.8, 0.8, 0.2)$ applied with probability 0.8, grayscaling with probability 0.2, and Gaussian blurring with probability 0.5. All experiments are conducted on NVIDIA V100 GPUs. To ensure statistical robustness, all reported results are averaged over 10 runs with different random seeds.

We adopt ResNet-18 and ResNet-50 as the backbone networks to extract visual features. Each backbone is followed by a linear projection head that maps the learned representations into a latent embedding space for contrastive learning. To implement our proposed Task Conflict Calibration (TC$^2$) module, we use a lightweight three-layer multi-layer perceptron (MLP), consisting of two hidden layers with ReLU activation and an output layer producing task-specific modulation weights. By default, we construct four self-supervised tasks within each mini-batch, enabling the model to capture diverse task-level variations and facilitating the analysis of task conflict. This configuration balances computational efficiency with effective task modeling.

\subsection{Unsupervised Learning}
\label{sec_app:unsupervised}

We employ the most commonly used SSL protocol~\cite{ermolov2021whitening} to train a supervised classifier based on the frozen feature extractors. In experiment, we use a ResNet-18 as a feature extractor for the first four datasets, and a ResNet-50 for other datasets. Several existing SSL methods are used as baselines, including  SimCLR~\cite{chen2020simple}, BarlowTwins~\cite{zbontar2021barlow}, BYOL~\cite{grill2020bootstrap}, SimSiam~\cite{chen2021exploring}, W-MES~\cite{ermolov2021whitening}, SwAV~\cite{caron2020unsupervised}, MoCo~\cite{he2020momentum},DINO~\cite{caron2021emerging}, SSL-HSIC~\cite{li2021self} and VICRegL~\cite{bardes2022vicregl},and C-JEPA~\cite{mo2024connecting}. The classification accuracies of a linear classifier (Top-1) and a 5-nearest neighbors (5-NN) classifier for evaluation on small datasets, and Top-5 classification accuracies for evaluation on other datasets.
Table~\ref{tab:self_un1} shows the results of unsupervised experiments on multiple datasets, where ``+TC$^2$'' means our proposed method with task conflict calibration. On small-scale datasets, BYOL and SimCLR improve by more than 3\% on CIFAR10 and Tiny ImageNet~\cite{le2015tiny} datasets. On large-scale datasets, SimCLR also improved over 4\% on the ImageNet-100 dataset~\cite{russakovsky2015imagenet}.
In summary, our method adds task information to baselines while mitigating the impact of task conflict, improving results in unsupervised learning.
In addition, we conducted comparative experiments with generative self-supervised learning methods, and the results are presented in Table~\ref{tab:gssl_1}. The results also demonstrate the effectiveness of our method.

To provide a fair comparison with existing multi-task learning approaches, we conducted experiments on the CIFAR-10 dataset, using the same task partitioning strategy described earlier. The baseline methods included PCGrad~\cite{yu2020gradient}, GradVac~\cite{wang2020gradient} and GradDrop~\cite{chen2020just}, and we evaluated performance using the Top-1 classification accuracy. For self-supervised learning, we employed SimCLR. Experimental results show that our method outperforms by 2.13\%, 1.71\%, and 1.33\%, respectively.

\subsection{Semi-supervised Learning}
\label{sec_app:semi-supervised}

Here, we adopt the most commonly used protocol for semi-supervised learning~\cite{zbontar2021barlow}. We create two balanced subsets by sampling 1\% and 10\% of the training dataset. We fine-tune all models on these two subsets for 50 epochs with different learning rates for the classifier and the backbone network. Specifically, we use 0.05 and 1.0 for the classifier, and 0.0001 and 0.01 for the backbone network, on the 1\% and 10\% subsets. 
Table~\ref{tab:semi} reports the classification results on ImageNet compared with existing methods using two pre-trained models. From the results, SimCLR increases about 2.3\%, and Barlow Twins increases about 3\% at the 1\% subset setting.

\subsection{Transfer Learning}
\label{sec_app:transfer}

\subsubsection{Video-based Task.}

We transfer pretrained models to a series of video tasks based on an evaluation platform called UniTrack~\cite{wang2021different}. UniTrack does not require additional training and, therefore, provides a more direct comparison between different representations. The video dataset we use contains SOT~\cite{wu2013online}, VOS~\cite{perazzi2016benchmark}, MOT~\cite{milan2016mot16}, MOTS~\cite{voigtlaender2019mots}, PoseTrack~\cite{andriluka2018posetrack}. All methods use the same ResNet-50 backbone and are evaluated based on UniTrack.

\subsubsection{Object Detection and Instance Segmentation.}

We adopt the most commonly used protocol for transfer learning~\cite{zbontar2021barlow}. We evaluate our framework on Pascal VOC and COCO datasets for object detection and instance segmentation. Specifically, we use Faster R-CNN~\cite{ren2015faster} with a C4-backbone~\cite{wu2019detectron2} for the VOC detection task, and use Mask R-CNN~\cite{he2017mask} with the same C4-backbone  and a 1$\times$ schedule for the COCO detection and segmentation tasks. In the fine-tuning stage, we search for an optimal learning rate on the target datasets and keep other parameters the same as in the Detection2 library. In the training stage, we train our Faster R-CNN model on the VOC 07+12 trainval set (16K images) and reduce the initial learning rate by a factor of 10 at 18K and 22K iterations. We also train Faster R-CNN on the VOC 07 trainval set (5K images), but with fewer iterations. For the Mask R-CNN model, we train it on the COCO 2017 train split and report the results on the val split.

\begin{table}[t]
    \centering
        \begin{tabular}{c|cccc}
        \toprule
        Method  & $n_f$=$d/4$ & $n_f$=$d/2$ & $n_f$=$d$ & $n_f$=$2d$ \\
        \midrule
        SimCLR+TC$^2$ & 90.78 & 92.53 & 93.35 & 93.01\\
        BYOL+TC$^2$ & 90.54 & 92.67 & 93.24 & 92.87\\
        \bottomrule
        \end{tabular}
\caption{The impact of the number of factors.}
\label{tab:abl_nf}    
\end{table}

\begin{table}[t]
    \centering
    \begin{tabular}{ccccc}
	\toprule
   $\lambda_s$  & 0.001 & 0.01 & 0.1 & 1 \\
    \hline
    SimCLR+TC$^2$ & 91.65 & 92.53 & 93.45 & 93.15\\
    MAE+TC$^2$ & 93.49 & 96.09 & 96.83 & 96.81 \\
    \hline
    \end{tabular}
    \caption{Parametric analysis of $\lambda_s$.}
    \label{tab:para_lambda_s}
\end{table}

\begin{table}[t]
    \centering
    \begin{tabular}{c|cccc}
    \toprule
    batch size  & 128 & 256 & 512 & 1024\\
    \midrule
    SimCLR+TC$^2$ & 91.78 & 92.12 & 93.12 & 94.84\\
    BYOL+TC$^2$ & 91.54 & 92.08 & 93.35 & 94.89\\
    \bottomrule
    \end{tabular}
\caption{The impact of the batch size.}
\label{tab:bs}
\end{table}

\subsection{Ablation Study Results}

\textbf{The number of generative factors $n_f$.} 
The dimension of the factor space spanned by the generative factors is the number of $n_f$. We analyze the impact of $n_f$ on CIFAR10 dataset. Specifically, we set $n_f$ to $d/4, d/2, d, 2d$. Experimental results in Table~\ref{tab:abl_nf} show that optimal results are obtained when the number of generative factors is the same as the dimension of the embedding space.

\textbf{The parameter $\lambda_s$.}
In this section, we conduct an experimental study of the model trade-off parameters. Our proposed method contain an important  parameter $\lambda_s$ to balance $\mathcal{L}_v$ and $\mathcal{L}_w$. Specifically, we vary in the range [0.001, 0.01, 0.1, 1], and use the SimCLR+TR method to record the classification accuracy of our method using ResNet-18 on the CIFAR-10 dataset. The results in Table \ref{tab:para_lambda_s} show that our method has small changes in accuracy under different parameter settings.

\textbf{Batch size}. We also investigate the impact of batch size on unsupervised learning performance. Specifically, we conduct experiments with batch sizes of 128, 256, 512, and 1024, and report the results in Table~\ref{tab:bs}. The results demonstrate that our method consistently improves model performance across different batch size settings, with relatively minor sensitivity to batch size variations.

\begin{table}
    \centering
    \begin{tabular}{c|cc}
    \toprule
    Method  & Top-1& 5-NN \\
    \midrule
    SimCLR & 91.80 & 88.42\\
    % SimCLR+T & 92.37 & 89.15\\
    SimCLR+TC$^2$ & \textbf{93.84} & \textbf{90.75}\\
    SimCLR+TC$^2$ w/o $\mathcal{L}_v$ & 91.85  & 88.58\\
    SimCLR+TC$^2$ w/o $\mathcal{L}_w$ & 90.79 & 87.65\\
    SimCLR+TC$^2$ w/o $\mathcal{L}_v$ and $\mathcal{L}_w$ & 88.75  & 84.37\\
    \midrule
    SimCLR+TC$^2$ w/o bi-level & 89.79 & 83.14 \\
    \bottomrule
    \end{tabular}
\caption{The impact of the loss function.}
\label{tab:abl_loss}
\end{table}

\begin{table}
    \centering
        \begin{tabular}{c|ccccc}
        \toprule
        Method  & $n_t$=2 & $n_t$=4 & $n_t$=8 & $n_t$=16 & $n_t$=32 \\
        \midrule
        SimCLR+TC$^2$ & 93.35 & 93.96 & 93.88 & 93.84 & 94.01\\
        BYOL+TC$^2$ & 93.24 & 94.23 & 94.18 & 93.89 & 93.27 \\
        \bottomrule
        \end{tabular}
\caption{The impact of the number of tasks.}
\label{tab:abl_task}
\end{table}

\subsection{Trade-Off Performance}
\label{sec_app:trade-off}
Considering that our proposed method is a plug-and-play method, we conduct experiments for the trade-off performance. We compare the trade-off performance of multiple baselines before and after using our TC$^2$ with the same backbone. Figure \ref{fig:ex_app_trade} illustrates the trade-off performance. The results indicate that incorporating the proposed TC$^2$ significantly improves performance while maintaining an acceptable computational cost and parameter size compared to the original SSL methods without introducing TC$^2$. 

\begin{figure}
    \centering
    \includegraphics[width=0.9\linewidth]{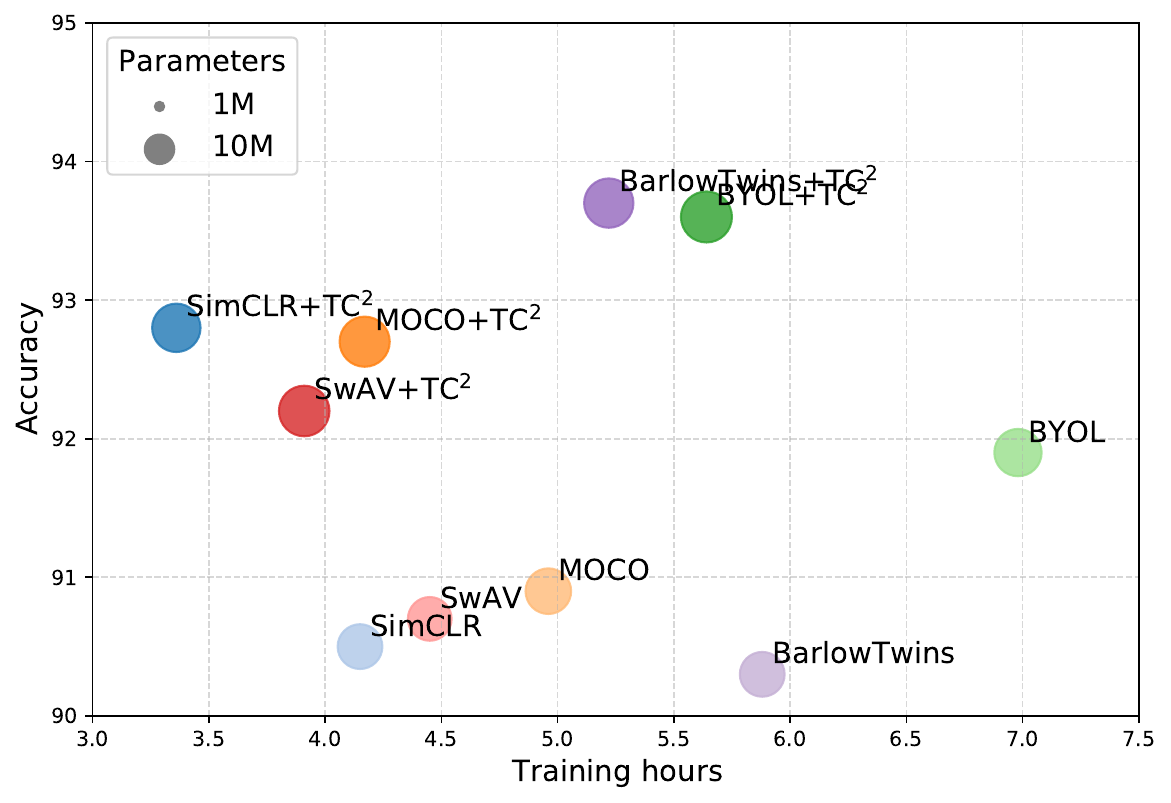}
    \caption{Trade-off performance.}
    \label{fig:ex_app_trade}
\end{figure}

\end{document}